\documentclass[final,onefignum,onetabnum]{siamart171218}
\usepackage[margin=1in]{geometry}

\usepackage{amsfonts}
\usepackage[english]{babel}
\usepackage[utf8]{inputenc}
\usepackage{graphicx}
\usepackage{amssymb}
\usepackage{mathrsfs}
\usepackage[colorinlistoftodos]{todonotes}
\usepackage{mathtools}
\usepackage{mathabx}
\usepackage{bm}
\usepackage{adjustbox}   
\usepackage{subfigure}
\usepackage{enumitem} 
\usepackage[mmddyyyy]{datetime}

\usepackage{tabularx,colortbl}
\usepackage{colortbl}
\usepackage{graphicx}

\colorlet{mylinkcolor}{blue}
\colorlet{mycitecolor}{orange}
\colorlet{myurlcolor}{orange}

\DeclareMathOperator{\vect}{vec}
\DeclareMathOperator{\tr}{tr}
\DeclareMathOperator{\spn}{span}

\newcommand{\B}[1]{\bm{#1}}
\newcommand{\vcr}[1]{\bm{#1}}
\newcommand{\mat}[1]{\bm{#1}}
\newcommand{\tsr}[1]{\pmb{\mathcal{#1}}}

\newcommand{\inti}[2]{\{{#1},\ldots, {#2}\}}

\newcommand{\R}{\mathbb{R}}

\usepackage[normalem]{ulem} 

\usepackage{cite}
\usepackage{amsmath,amssymb,amsfonts}

\usepackage{algpseudocode}
\usepackage{algorithm}

\usepackage{graphicx}
\usepackage{textcomp}
\usepackage{xcolor}

\usepackage{listings}
\definecolor{mygreen}{rgb}{0,0.2,0}
\definecolor{mygray}{rgb}{0.5,0.5,0.5}
\definecolor{mymauve}{rgb}{0.58,0,0.82}
\definecolor{mypurple}{rgb}{0.38,0,0.32}
\definecolor{myblue}{rgb}{0.1,0,0.32}
\newcommand{\costyle}{\footnotesize\ttfamily\bfseries}
\newcommand{\kwstyle}{\costyle\textcolor{myblue}}
\def\BibTeX{{\rm B\kern-.05em{\sc i\kern-.025em b}\kern-.08em
    T\kern-.1667em\lower.7ex\hbox{E}\kern-.125emX}}
\begin{document}

\title{ Alternating Mahalanobis Distance Minimization for Stable and Accurate CP Decomposition 
}

\author{Navjot Singh\thanks{Department of Computer Science, University of Illinois at Urbana-Champaign, Urbana, IL, 61801 (\email{navjot2@illinois.edu}, \email{solomon2@illinois.edu}).} \and Edgar Solomonik\footnotemark[1] .}

\maketitle

\begin{abstract}
CP decomposition (CPD) is prevalent in chemometrics, signal processing, data mining and many more fields. While many algorithms have been proposed to compute the CPD, alternating least squares (ALS) remains one of the most widely used algorithm for computing the decomposition. Recent works have introduced the notion of eigenvalues and singular values of a tensor and explored applications of eigenvectors and singular vectors in areas like signal processing, data analytics and in various other fields. We introduce a new formulation for deriving singular values and vectors of a tensor by considering the critical points of a function different from what is used in the previous work. Computing these critical points in an alternating manner motivates an alternating optimization algorithm which corresponds to alternating least squares algorithm in the matrix case. However, for tensors with order greater than equal to $3$, it minimizes an objective function which is different from the commonly used least squares loss. Alternating optimization of this new objective leads to simple updates to the factor matrices with the same asymptotic computational cost as ALS. We show that a subsweep of this algorithm can achieve a superlinear convergence rate for exact CPD with known rank and verify it experimentally. We then view the algorithm as optimizing a Mahalanobis distance with respect to each factor with ground metric dependent on the other factors. This perspective allows us to generalize our approach to interpolate between updates corresponding to the ALS and the new algorithm to manage the tradeoff between stability and fitness of the decomposition. Our experimental results show that for approximating synthetic and real-world tensors, this algorithm and its variants converge to a better conditioned decomposition with comparable and sometimes better fitness as compared to the ALS algorithm.
\end{abstract}

%
%

\section{Introduction}
\label{sec:intro}
The canonical polyadic or CANDECOMP/PARAFAC (CP) tensor decomposition~\cite{hitchcock1927expression,harshman1970foundations} is used for analysis and compression of multi-parameter datasets, and prevalent in tensor methods for scientific simulation~\cite{murphy2013fluorescence,sidiropoulos2017tensor,maruhashi2011multiaspectforensics,cong2015tensor}. 
For an order $3$ tensor $\tsr{T}$, a rank $R$ CP decomposition is
   \begin{align*}
    \tsr{T} = [\![ \mat{A}, \mat{B}, \mat{C} ]\!], \quad
    t_{ijk} = \sum_{r=1}^R a_{ir} b_{jr} c_{kr}.
    \end{align*}
Determining the CP rank or finding an approximate CP decomposition of a tensor, so as to minimize,
\begin{align}
\label{eq:CP_objective}
    f(\mat A, \mat B, \mat C) =
\frac{1}{2}\Big\|\tsr{T} - [\![ \mat{A}, \mat{B}, \mat{C} ]\!]\Big\|^2_F,
\end{align}
are NP-hard problems~\cite{Hillar:2013}. The CP decomposition of a tensor can be computed via various optimization algorithms, such as alternating least squares~\cite{karlsson2016parallel,hayashi2017shared,schatz2014exploiting, battaglino2017practical} which aims to minimize the objective~\eqref{eq:CP_objective} in an alternating manner by considering all except one factor matrix fixed. There have been several attempts to improve the performance of ALS algorithm by considering it's variations~\cite{rajih2008enhanced,nion2008enhanced,mitchell2018nesterov,ma2018accelerating}. Several methods which aim to minimize~\eqref{eq:CP_objective} with respect to all the factor matrices use gradient-based information~\cite{acar2011scalable,paatero1997weighted,phan2013low,sorber2013optimization,tichavsky2013further,tomasi2006comparison} to update all the factors simultaneously. Another set of methods  optimize for all the factors simultaneously by formulating~\eqref{eq:CP_objective} as a nonlinear least squares problem by considering the entries of all the factors as variables. In addition to gradient information, these iterative methods use second order information to compute the next step which requires a system solve~\cite{paatero1997weighted,tomasi2005parafac}, and can be achieved by an implicit conjugate gradient algorithm~\cite{sorber2013optimization,singh2021comparison}.

Tensor eigenvalue problems are relevant in the context of solving multilinear systems, simulating quantum systems, exponential data fitting and many other application areas~\cite{qi2018tensor}. However, the study of tensor eigenvalues and tensor singular values is at a relatively nascent stage,~\cite{lim2005singular,li2013z} provide a definition and introduction to eigenvalues and singular values of a tensor. Computing eigenvalues of a tensor is a hard problem, and can be solved via iterative methods for special cases such as computing a subset of eigenvalues of a tensor~\cite{kolda2011shifted} or computing the real eigenvalues pairs of a real symmetric tensor~\cite{cui2014all}. The tensor eigenvalue problem is motivated by applications like blind source separation~\cite{blindsource}, independent component analysis (ICA)~\cite{hyvarinen1999survey} which also motivate a closely related problem of diagonalizing a tensor. The concept of tensor diagonalization was introduced in~\cite{comon1994tensor}, where approximate diagonalization of the tensor is considered by minimizing the sum of squares of off-diagonal entries or maximizing the sum of squares of diagonal entries of the tensor. There have been many follow up works~\cite{liang2019further,li2018globally,usevich2020approximate,tichavsky2016nonorthogonal} which consider approximate diagonalization of the tensor by invertible and orthogonal transformations.

In this work, we introduce a formulation for computing the singular values and vectors of a tensor by considering a logarithmic penalty function instead of Lagrangian variables~\cite{lim2005singular} and computing the critical points of this function. This formulation when generalized to computing invariant subspaces of a matrix, leads to diagonalization of the matrix and can be linked to the singular values and vectors of the matrix. When extended to tensors with order greater than or equal to $3$, this formulation leads to another notion of diagonalization of the tensor which is different from the one introduced in prior work. The critical points of this new function spectrally diagonalize the tensor, i.e., the  transformed equidimensional tensor of mode length $R$ has $R$ elementary eigenvectors with unit eigenvalues. Computing these stationary points alternatively motivates an alternating optimization algorithm for computing the CP decomposition of a tensor.

\subsection{Motivation: Eigenvectors via Lagrangian Optimization}
In the case of low-rank matrix approximation, the Eckart-Young-Mirsky theorem shows that the best low-rank approximation may be obtained from the singular value decomposition (SVD).
This connection relates low-rank factors to critical points of the bilinear form, \(f(\B x, \B y) = \B x^T \mat{A}\B y\) with $\|\B x\| \neq 0$, $\|\B y\| \neq 0$.
For tensors of order 3 and higher, tensor singular values have been similarly derived from critical points of multilinear forms.
In particular, Lim~\cite{lim2005singular} derives singular vectors and singular values by imposing constrains $\|\B x\| = \|\B y \|=1$ and considering the critical points of the Lagrangian function.
The same results may be obtained by instead considering a logarithmic interior point barrier function for the constraints, $\|\B x\| \neq 0$, $\|\B y\| \neq 0$, so
\[f(\B x, \B y) = \B x^T \mat{A}\B y - \log(\|\B x\|\|\B y\|), \quad \nabla f(\B x,\B y) = \B 0 \ \Rightarrow\ \mat{A}\B y = \B x/\|\B x\|^2, \ \mat{A}^T\B x = \B y/\|\B y\|^2.\]
Consequently, with $\sigma = 1/(\|\B x\|\|\B y\|)$ and $\B u = \B x / \|\B x\|$, $\B v = \B y / \|\B y\|$, we have $\mat{A} \B v = \sigma \B u$ and $\mat A^T \B u = \sigma \B v$.
The use of a coefficient for the barrier function only affects the scaling of any critical point vectors, $\B x$ and $\B y$.
For tensors of order 3 and higher, tensor singular values can be similarly derived from critical points of 
\begin{align*}
f(\B x^{(1)}, \ldots, \B x^{(N)}) &= \sum_{i_1\ldots i_N}t_{i_1\ldots i_N}x^{(1)}_{i_1} \cdots x^{(N)}_{i_N}  - \log(\|\B x^{(1)}\|\cdots \|\B x^{(N)}\|), \\
\nabla f(\B x^{(1)}, \ldots, \B x^{(N)}) &= \B 0 \ \Rightarrow\ \frac{x^{(j)}_{i_j}}{\|\B x^{(j)}\|^2} = \sum_{i_1\ldots \hat{i}_j \ldots i_N} t_{i_1\ldots i_N}x^{(1)}_{i_1} \cdots \hat{x}^{(j)}_{i_j} \cdots x^{(N)}_{i_N},
\end{align*}
where $i\ldots \hat{j}\ldots k$ implies $j$ is omitted from the sequence.
The use of 2-norm in the above definitions leads to $l^2$ singular vectors~\cite{lim2005singular} and with a symmetric tensor and each $\B x^{(i)}=\B x^{(j)}$ it yields Z-eigenvectors~\cite{li2013z}.
Choosing another vector norm in $\{3,\ldots, N\}$ leads to other notions of singular vectors and eigenvectors~\cite{li2013z}.

The only significant known correspondence between tensor singular vectors or eigenvectors and the CP decomposition, is in the case of a rank $R=1$ CP.
In this case, the singular vector with the largest singular value and the best rank-1 approximation coincide (for symmetric tensors, these also correspond to the largest eigenvalue tensor eigenvector).
The rank-1 approximation problem is also NP-hard for tensors of order 3 and higher~\cite{Hillar:2013}.
In this work, motivated by an efficient iterative scheme, we consider an extension of the variational notion of a single singular vector tuple to many.

\subsection{Tensor Spectral Diagonalization via Lagrangian Optimization}
\label{subsec:spectral_lagrange}

We denote an inner product of matrices as $\langle \B X, \B Y \rangle = \langle \vect(\mat X), \vect(\mat Y) \rangle$, and similar for tensors $\tsr X$, $\tsr Y$. 
The invariant subspaces of a matrix $\mat A$ may be obtained by considering the critical points of a generalization of $\B x^T \mat  A \B y = \langle \mat A, \B x \B y^T\rangle$ to the matrix case,
\[f(\B X, \B Y) = \langle \mat A, \mat X \mat Y^T \rangle, \text{ s.t. }\ \det(\mat X^T \mat X)\neq 0, \det(\mat Y^T \mat Y) \neq 0.\]
Transforming the inequality constraint into a logarithmic barrier function, we obtain
\begin{align}
\mathcal{L}_f(\B X, \B Y) &= \langle \mat A, \mat X \mat Y^T \rangle - \frac{1}{2}(\log(\det(\mat X^T \mat X)) - \log(\det(\mat Y^T \mat Y)))\\ 
&= \tr(\B X^T \mat{A}\B Y) - \frac{1}{2}\tr(\log(\B X^T \B X \B Y^T\B Y)). \label{eq:mat_crit}
\end{align}
The critical points of $\mathcal L_f$ satisfy,
\[\mat A \mat Y \mat X^T \cong \mat I \text{ and } \mat A^T \mat X \mat Y^T \cong \mat I.\]
At a critical point of $f(\B X, \B Y)$, the column span of $\B X$, $\spn\{\B x_{1}, \ldots \B x_{n}\}$, must be an invariant subspace of $\B A \B A^T$, while the columns of $\B Y$ span an invariant subspace of $\B A^T\B A$. 
These critical points diagonalize $\mat{A}$ in the sense that $\B X^T \mat{A}\B Y = \B I$.
In the tensor case, a critical point $(\mat X^{(1)},\ldots, \mat X^{(N)})$ of
\begin{align}
f(\mat X^{(1)},\ldots, \mat X^{(N)}) &= 
\langle \tsr T,[\![ \mat{X}^{(1)}, \ldots \mat{X}^{(N)} ]\!]\rangle, \text{ s.t. } \det(\mat X^{(n)}{}^T\mat X^{(n)}) \neq 0, \forall n \in \{1,\ldots, N\} \nonumber, \\
\mathcal{L}_f(\mat X^{(1)},\ldots, \mat X^{(N)}) &= \sum_{r=1}^R \sum_{i_1 \ldots i_N} t_{i_1\ldots i_N}x^{(1)}_{i_1r}\cdots x^{(N)}_{i_Nr} - \frac{1}{2}\tr(\log({\B X^{(1)}}^T\B X^{(1)}\cdots {\B X^{(N)}}^T\B X^{(N)})) 
\label{eq:f_tsr}
\end{align}
gives invariant subspaces of the tensor in the sense that, $\forall i\in\{1,\ldots,N\}$,
\[\forall \B v \in \spn\{\bigotimes_{j\neq i}\B x^{(j)}_{1},\ldots,\bigotimes_{j\neq i}\B x^{(j)}_{R}\}, 
\quad \B T_{(i)}\B v \in
\spn\{\B x^{(i)}_1,\ldots, \B x^{(i)}_R\},\]
where $\B T_{(i)}$ is the mode-$i$ matricization (unfolding) of the tensor $\tsr T$. 

Since the reconstructed tensor $\tilde{\tsr{T}}$, where \(\tilde{\tsr{T}} = [\![\mat Y^{(1)}, \ldots {\mat Y^{(N)}} ]\!],\; \mat Y^{(n)} = {\mat X^{(n)}{}^\dagger}^T,\; \forall n \in \{1,\ldots, N\} \), captures the action of $\B T_{(i)}$ on an invariant subspace, the application of each matricization may be performed with bounded backward error. In Section~\ref{subsec:Alternating_mahalanobis}, we show that the backward error is bounded by $\|\mat T_{(i)}\vcr z^\perp \|$, where $\vcr z^\perp$ is the projection of $\vcr z$ onto the orthogonal complement of column span of $\bigodot_{j=1, j\neq n}^N \mat X^{(j)}$. We show that this bound also holds for ALS, and in general for a family of algorithms based on alternating minimization of Mahalanobis distance~\cite{chandra1936generalised} between the input and reconstructed tensor.

Another observation regarding the critical points of~\eqref{eq:f_tsr} is that each matricization of the tensor reconstructed from a critical point, $\tsr{X} = [\![\mat X^{(1)}, \ldots {\mat X^{(N)}} ]\!]$ is a right inverse of the corresponding matricization of the input tensor $\tsr{T}$ when CP rank is equal to the mode length and more generally,
\[
\mat{T}_{(i)} \mat{X}_{(i)}^{T} = \mat \Pi^{(i)}, \quad \forall i \in \{1,\ldots,N\},
\]
where each $\mat \Pi^{(i)}$ is a projector onto the column space of $\mat X^{(i)}$. 
This $\tsr{X}$ satisfies some but not all of the properties of previously proposed generalizations of the Moore-Penrose inverse to tensors~\cite{sun2016moore,liang2019further}.

Further, the critical point gives a transformation that spectrally diagonalizes $\tsr T$,
\[z_{j_1\ldots j_N} = \sum_{i_1\ldots i_N} t_{i_1\ldots i_N}x_{i_1j_1}^{(1)}\cdots x_{i_Nj_N}^{(N)},\]
so that $\tsr Z$ has $R$ eigenvectors that are elementary vectors with unit eigenvalues (for any tensor eigenvector/eigenvalue definition, i.e., $l^p$ eigenvector for any choice of $p$~\cite{li2013z}), since
\[z_{j_k j_k\dots j_p j_k \dots j_k}= \delta_{j_kj_p} \text{ for all }  p\neq k \in \{1,\dots,N\}.\]
Beyond these properties, we show that ${\B X^{(1)}{}^\dagger}^T,\ldots,{\B X^{(N)}{}^\dagger}^T$ may be used to obtain an exact CP decomposition or an effective low-rank approximate CP decomposition.

\subsection{Alternating Optimization Method}

The critical points defined above may be computed efficiently by a method similar to ALS.
ALS solves a set of overdetermined linear equations at each step to minimize Frobenius norm error relative to one factor, e.g., it solves for $\mat A$ in
\[(\mat C \odot \mat B) \mat A^T \cong \mat{T}_{(1)}^T.\]
The update rule for each subproblem, may be written as a product of the pseudoinverse of the Khatri-Rao product of two factors and an unfolding of the tensor, i.e.,
\[\mat{A} = \mat{T}_{(1)}(\mat{C}\odot \mat{B})^{\dagger}{}^T.\]
Some of the major advantages of the ALS algorithm is its guaranteed monotonic decrease in residual, low per-iteration computational cost, and amenability to parallelization. It has been shown that ALS achieves linear local convergence to minima of the CP residual norm~\cite{uschmajew2012local}.

To obtain a critical point in the high-order tensor function~\eqref{eq:f_tsr}, we propose a different alternating update scheme, which finds the solution $\mat U$ to the linear least squares problem,
\[(\mat T_{(1)} (\mat W \odot \mat V))\mat U^T \cong \mat I.\]
With $\mat{A}^T=\mat U^\dagger$, $\mat{B}^T = \mat V^\dagger$, and $\mat C^T = \mat W^\dagger$, we observe that the update is similar to that of ALS,
\[\mat{A} = \mat{T}_{(1)}(\mat{C}^{\dagger}{}^T \odot \mat{B}^\dagger{}^T).\]
A stationary point of this alternating update scheme provides a critical point of~\eqref{eq:f_tsr}.

\subsection{Convergence Results}

The new alternating update scheme is highly effective at finding an exact CP decomposition, if one exists.
In particular, we show that the method achieves a superlinear rate of local convergence to exact CP decompositions. In Section~\ref{sec:conv}, we prove that the convergence order is $\alpha$ per subproblem or $\alpha^N$ per sweep of alternating updates, where $\alpha$ is the unique real root of the polynomial $x^{N-1} - \sum_{i=0}^{N-2}x^i$.
For $N=3$, $\alpha=(1+\sqrt{5})/2$, while for higher $N$, $\alpha$ increases.
A superlinear convergence rate for CP decomposition is also achievable via general optimization algorithms such as Gauss-Newton~\cite{singh2021comparison, sorber2013optimization}.
However, the alternating optimization scheme we propose has a much lower per iteration cost (about the same as ALS, which achieves only linear convergence).

For a given tensor and any choice of rank, the critical points are generally not unique (as in the case of matrices).
Theoretical characterization of the conditions under which critical points of~\eqref{eq:f_tsr} exist in this scenario remains an open problem as it requires proving existence of roots of a system of nonlinear equations that have the same number of variables and equations. Note that the problem of proving if the best CP rank approximation exists also requires existence of a solution of system of nonlinear equations. The best CP rank approximation may not exist, which has lead to the notion of border rank~\cite{kolda2009tensor}.
Consequently, establishing existence of critical points for our scenario is likely also nontrivial.
Therefore, with the assumption that a critical point exists, we show in Section~\ref{subsec:conv_approx} that the proposed iterative scheme achieves local convergence to it (in this case, at a linear rate). 

We perform numerical experiments in Section~\ref{sec:exp} to confirm the rate of convergence of the algorithm for computing CP decomposition of different tensors with known rank. The observed rate of convergence values agree with the theoretical rate of convergence with an error of about $0.2\%$ for order $3$ and an error of about $1\%$ for order $4$ tensors. The experiments also confirm the  convergence analysis of the algorithm to stationary points for approximate decomposition of tensors as described in Lemma~\ref{lem:err_approx_mnorm}

\subsection{Generalizations and Experimental Evaluation}

The proposed algorithm may also be used for approximate CP decomposition, but does not minimize the Frobenius norm of the residual directly.
Instead, when optimizing for $\mat{A}$, the algorithm minimizes
\[\big\|(\mat I \otimes \mat{B}^\dagger \otimes \mat{C}^\dagger)\vect(\tsr{T} - [\![ \mat{A}, \mat{B}, \mat{C} ]\!])\big\|_F.\]
For a fixed residual error in the decomposition, the magnitude of this error metric will generally depend on the conditioning of $\mat{A}$, $\mat{B}$, and $\mat{C}$.
Hence, this alternating minimization procedure tends to converge to well-conditioned factors (and well-conditioned CP decompositions~\cite{vannieuwenhoven2017condition,breiding2018condition}).

We generalize this method by considering a Mahalanobis distance between the input and reconstructed tensor. The original motivation for Mahalanobis distance minimization of tensors came from the work on minimization of Wasserstein distance between tensors for nonnegative CP decomposition~\cite{afshar2021swift}. This generalization allows us to reformulate each update to a factor of the above introduced algorithm as a minimizer of a Mahalanobis distance~\cite{chandra1936generalised} with the ground metric dependent on the other remaining factors. This reformulation helps extend the introduced algorithm to any CP rank by using the same ground metric. Moreover, we are able to interpolate between the introduced algorithm and ALS by interpolating the ground metric. Our experiments in Section~\ref{sec:exp} suggest that the interpolated updates help manage the trade-off between fitness and conditioning of the decomposition. We measure conditioning of the decomposition by computing the normalized CPD condition number~\cite{breiding2018condition}. The condition number can be computed in an efficient manner for decompositions with small CP rank by compressing the matrix for which the smallest singular value needs to be computed and henceforth reducing the computational cost significantly as described in Appendix~\ref{sec:app}. By using this efficient approach, we are able to track condition number of the decomposition at each iteration of the algorithms. For synthetic as well as real-world tensors, we observe that by utilizing hybrid updates of the introduced algorithm, we can find decompositions with a condition number lower by a factor as large as $10^{4}$ with a change in fitness of about $0.01$ only when compared to ALS. 

\section{Background}
\label{sec:bg}
We introduce the notation and definitions used in the subsequent sections here along with a brief introduction to the alternating least squares algorithm for computing CP decomposition~\cite{harshman1970foundations, carroll1970analysis}.
\subsection{Notation and Definitions}
We use tensor algebra notation in both element- wise form and specialized form for tensor operations~\cite{kolda2009tensor}.
For vectors, bold lowercase Roman letters are used, e.g., $\vcr{x}$. For matrices, bold uppercase Roman letters are used, e.g., $\mat{X}$. For tensors, bold calligraphic fonts are used, e.g., $\tsr{X}$. An order $N$ tensor corresponds to an $N$-dimensional array with dimensions $s_1\times \cdots \times s_N$. 
Elements of vectors, matrices, and tensors are denoted in subscript, e.g., $x_i$ for a vector $\vcr{x}$, $x_{ij}$ for a matrix $\mat{X}$, and $x_{ijkl}$ for an order 4 tensor $\tsr{X}$. 
The $i$th column of a matrix $\mat{X}$ is denoted by $\vcr{x}_i$. 
The mode-$n$ matrix product of a tensor $\tsr{X} \in \mathbb{R}^{s_1 \times \cdots \times s_N}$ with a matrix $ \mat{A}\in \mathbb{R}^{J\times s_n}$ is denoted by $\tsr{X}\times_n \mat{A}$, with the result having dimensions $s_1\times\cdots\times s_{n-1}\times J\times s_{n+1}\times\cdots\times s_N$. 
Matricization is the process of reshaping a tensor into a matrix. Given a tensor $\tsr{X}$ the mode-$n$ matricized version is denoted by $\mat{X}_{(n)}\in \mathbb{R}^{s_n\times K}$ where $K=\prod_{m=1,m\neq n}^N s_m$. 
We use parenthesized superscripts as labels for different tensors and matrices, e.g., $\mat{A}^{(1)}$ and $\mat{A}^{(2)}$ are different matrices.

The Hadamard product of two matrices $\mat{U}, \mat{V} \in \mathbb{R}^{I\times J}$ resulting in matrix $\mat{W} \in \mathbb{R}^{I\times J}$ is denoted by $\mat{W} = \mat{U} \ast \mat{V}$, where $w_{ij}= u_{ij}v_{ij}$. 
The inner product of matrices $\mat{U}, \mat{V}$ is denoted by $\langle \mat{U}, \mat{V} \rangle = \sum_{i,j}u_{ij}v_{ij}$.
The outer product of K vectors $\vcr{u}^{(1)}, \ldots , \vcr{u}^{(K)}$ of corresponding sizes $s_1, \ldots , s_K$ is denoted by $\tsr{X} = \vcr{u}^{(1)} \circ \cdots \circ \vcr{u}^{(K)}$ where $\tsr{X} \in \mathbb{R}^{s_1 \times \cdots \times s_K}$ is an order $K$ tensor. 

For matrices $\mat{A}\in \mathbb{R}^{I\times K} = \begin{bmatrix} \vcr{a}_1, \ldots, \vcr{a}_K \end{bmatrix}$ and $\mat{B}\in \mathbb{R}^{J\times K}= \begin{bmatrix} \vcr{b}_1, \ldots, \vcr{b}_K \end{bmatrix}$, their Khatri-Rao product resulting in a matrix of size $(IJ)\times K$ defined by
$\mat{A}\odot \mat{B} = [\vcr{a}_1\otimes \vcr{b}_1,\ldots, \vcr{a}_K\otimes \vcr{b}_K]$, where $\vcr{a}\otimes \vcr{b}$ denotes the Kronecker product of the two vectors. We define the Mahalanobis norm for a matrix $\mat A$ with ground metric $\mat M$ as $\|\mat A\|_{\mat M} = \sqrt{\vect(\mat A)^T \mat M \vect(\mat A)}$ and similarly for tensor $\tsr{T}, \|\tsr{T}\|_{\mat M} = \vect(\tsr{T})^T \mat M \vect( \tsr{T})$. To ease the notation for N khatri Rao products, we use  $\bigodot_{n=1}^N = \mat A^{(N)} \odot \ldots \odot \mat A^{(1)}$ and similarly $\bigotimes_{n=1}^N \mat A^{(n)} = \mat A^{(N)} \otimes \ldots \otimes \mat A^{(1)}$,$\bigast_{n=1}^N \mat A^{(n)} = \mat A^{(N)} \ast \ldots \ast \mat A^{(1)}$.  We use  $\sigma_{\text{min}}(\mat P)$ to denote the minimum singular value of the matrix $\mat P$.

\subsection{Alternating least squares for CP decomposition}
The CP tensor decomposition~\cite{hitchcock1927expression,harshman1970foundations} for an input tensor $\tsr{X}\in \R^{I_1\times\dots\times I_N}$
is denoted by
   \[
    \tsr{X} \approx [\![ \mat{A}^{(1)}, \cdots , \mat{A}^{(N)} ]\!], \quad \text{where} \quad  
    \mat{A}^{(i)} = [ \mat{a}_1^{(i)}, \cdots , \mat{a}_r^{(i)} ],
    \]
and serves to approximate a tensor by a sum of $R$ tensor products of vectors,
  \[
  \tsr{X} \approx \sum_{r=1}^{R} \mat{a}_r^{(1)}\circ \cdots \circ \mat{a}_r^{(N)}.
  \]
It is sometimes useful to normalize the factor matrices so that each column of the factors has a unit 2-norm and
the weights are absorbed into a vector $\vcr{\lambda} \in  \mathbb{R}^{R}$, given as,
\[\tsr{X} \approx \sum_{r=1}^{R} \mat{a}_r^{(1)}\circ \cdots \circ \mat{a}_r^{(N)} = \sum_{r=1}^{R} \lambda_{r} \bar{\mat{a}}_r^{(1)}\circ \cdots \circ \bar{\mat{a}}_r^{(N)},
\]
where $\bar{\mat A}^{(n)}$ are column normalized for all $n$ and denoted as \[\tsr X \approx [\![ \mat \Lambda;\bar{\mat{A}}^{(1)}, \cdots , \bar{\mat{A}}^{(N)} ]\!],\] where $\mat \Lambda$ is a diagonal matrix with $\vcr{\lambda}$ on the diagonal.
The CP-ALS method aims to minimize the nonlinear least squares problem
   \begin{equation}
f(\mat{A}^{(1)}, \ldots , \mat{A}^{(N)}) :=
 \frac{1}{2}\Big\vert \Big\vert\pmb{\mathcal{X}}-[\![ \mat{A}^{(1)}, \cdots , \mat{A}^{(N)} ]\!]\Big\vert \Big\vert_F^2,
\label{eq:obj}
\end{equation}
by alternatively minimizing a sequence of least squares problems for each of the factor matrices $\mat A^{(n)}$. This results in linear least squares problems for each row,
\[
 \mat{A}^{(n)}_{\text{new}}\mat{P}^{(n)}{}^T \cong \mat{X}_{(n)},
\]
where the matrix $\mat{P}^{(n)}\in \mathbb{R}^{D_n \times R}$, where $D_n = \prod_{i=1,i\neq n}^{N}I_i$ is formed by Khatri-Rao products of the other factor matrices,
\begin{align}
    \mat{P}^{(n)}=\bigodot_{m=1,m\neq n}^N\mat{A}^{(m)}. \label{eq:P}
\end{align}
These linear least squares problems are often solved via the normal equations~\cite{kolda2009tensor},
     \[
     \mat{A}^{(n)}_{\text{new}}\boldsymbol{\Gamma}^{(n)}\leftarrow \mat{X}_{(n)}\mat{P}^{(n)},
     \]  
where $\mat{\Gamma}\in\R^{R\times R}$ can be computed via
\begin{equation}
\boldsymbol{\Gamma}^{(n)}= \bigast_{i=1,i\neq n}^N\mat S^{(i)},
\label{eq:gamma}
\end{equation}
with each
$\mat{S}^{(i)} = \mat{A}^{(i)T}\mat{A}^{(i)}.$
The \textit{Matricized Tensor Times Khatri-Rao Product} or \mbox{MTTKRP} computation $\mat{M}^{(n)}=\mat{X}_{(n)}\mat{P}^{(n)}$ is the main computational bottleneck of CP-ALS\cite{ballard2018communication}. 
For a rank-R CP decomposition, this computation has the cost of $O(I^NR)$ if $I_n=I$ for all $n\in\inti{1}{N}$. There have been various developments to optimize computation of MTTKRP, like dimension-tree algorithm~\cite{phan2013fast,vannieuwenhoven2015computing,kaya2016parallel,ballard2018parallel,kaya2017high,kaya2019computing} for dense tensors and sparse MTTKRP~\cite{choi2018blocking} for sparse tensors.


\section{Basic Description of the New Algorithm}
\label{sec:alg}
\begin{algorithm}[]
\begin{algorithmic}[1]
\State{\textbf{Input: }Tensor $\tsr{X}\in\mathbb{R}^{I_1\times\cdots\times I_N}$, rank $R$}
\State{Initialize $\{\mat{A}^{(1)}, \ldots , \mat{A}^{(N)}\}$ so each $\mat{A}^{(n)}\in\mathbb{R}^{I_n\times R}$ is random}

\While{not converged}
    \For{\texttt{$n \in \inti{1}{N} $}}
      \State{\(\mat{A}^{(n)} = \mat{X}^{(N)}_{(n)}\bigg(\bigodot_{m=1,m\neq n}^N \mat{A}^{(m)}{}^\dagger{}^T\bigg)\)}
    \EndFor
\EndWhile
\State{\Return factor matrices $\{\mat{A}^{(1)}, \ldots , \mat{A}^{(N)}\}$ }
\end{algorithmic}
\caption{Basic description of the new alternating update scheme.}
\label{alg:mnorm_basic}
\end{algorithm}
We first provide a complete description of the alternating update scheme proposed in the introduction.
To compute the decomposition of a tensor $\tsr{X}$, the algorithm maintains a CP decomposition given by
\[[\![ \mat{A}^{(1)}, \cdots , \mat{A}^{(N)} ]\!],\]
and updates each $\mat{A}^{(n)}$ in an alternating manner,
\begin{align}
\label{eq:basic_mnorm_update}
    \mat{A}^{(n)} = \mat{X}_{(n)}\bigg(\bigodot_{m=1,m\neq n}^N \mat{A}^{(m)}{}^\dagger{}^T\bigg).
\end{align}
This update may be written in elementwise form in terms of the pseudoinverses $\mat{U}^{(m)} = \mat{A}^{(m)}{}^\dagger$ as
\[a^{(n)}_{i_nr} = \sum_{i_1\ldots \hat{i}_{n}\ldots i_N} x_{i_1\ldots i_N} \prod_{m=1,m\neq n}^N u^{(m)}_{ri_m}.\]

Algorithm~\ref{alg:mnorm_basic} details each sweep of such updates.
Like with the ALS, it is advisable to recalibrate the norms of the columns of each factor before the subsweep corresponding to $nth$ factor so that $\|\vcr{a}^{(k)}_r\|=1$ for all $k \neq n$ and $r$. 
With this recalibration, a valid convergence criteria is to check whether the magnitude of change in the factors exceeds a predefined threshold at each sweep.
The method is invariant to the rescaling in exact arithmetic, but this calibration helps reduce the effects of round-off error. This calibration is also cost efficient, since pseudoinverse of only one matrix changes per subweep. This makes the algorithm accessible to all the optimizations involved in computing \textit{Matricized Tensor Times Khatri-Rao Product} or \mbox{MTTKRP} in each subsweep of ALS such as the dimension tree algorithm~\cite{phan2013fast,vannieuwenhoven2015computing,kaya2016parallel,ballard2018parallel,kaya2017high,kaya2019computing}.

\subsection{Cost Analysis}

The cost of each sweep of Algorithm~\ref{alg:mnorm_basic} corresponds to the cost of computing the pseudoinverse of each factor, as well as a set of $N$ MTTKRP operations.
A dimension tree or multi-sweep dimension tree~\cite{ma2021efficient} may be used to compute the set of MTTKRPs in the same way as done in the alternating least squares algorithm.
The overall per-sweep cost with a multi-sweep dimension tree is then given by
\[\frac{2N}{N-1} \prod_{n=1}^NI_n R+ O\Big((\sum_{n=1}^NI_n)R^2\Big).\]
The cost of an ALS sweep with a multi-sweep dimension tree is 
\[\frac{2N}{N-1} \prod_{n=1}^NI_n R+ O(NR^3),\]
which is less expensive as solving an overdetermined system via normal equations is cheaper than computing the pseudoinverse of a matrix. If $\tsr{X}$ is sparse, the method can benefit from existing work on efficiently performing MTTKRP with a sparse tensor~\cite{choi2018blocking} and therefore has the same leading order cost per-sweep as that of ALS for sparse tensors as well.

\section{Convergence Rate for Exact Decomposition}
\label{sec:conv}
In this section, we theoretically analyze the asymptotic rate of local convergence of the Algorithm~\ref{alg:mnorm_basic} for when an exact CP decomposition of rank $R$ less than equal to the length of all modes of the tensor exists. To derive the rate of convergence for an exact decomposition, we relate the distance between the computed factor in each subproblem and the true factor, to the error in the other factor matrices.
The following lemma states the error in computing $\mat{A}^{(N)}$, but can be trivially extended to any $\mat{A}^{(n)}$ for $n \in \{1,\dots,n\}$.
We consider the error in the normalized factor and the error in the magnitude of CPD components modulo column scaling. 
\begin{lemma}\label{lem:err_norm_mnorm}
Suppose $\tsr{X} = [\![ \mat D ;\mat{A}^{(1)}, \ldots, \mat{A}^{(N)}]\!]$, where each $\mat{A}^{(i)}\in\mathbb{R}^{s_i\times R}$ with $s_i\geq R$ is full rank and has normalized columns, i.e., $\|\vcr{a}^{(i)}_j\|_2 = 1$ for all $i,j$ and $\bar{\mat{A}}^{(n)} = \mat{A}^{(n)} + \mat{\Delta}^{(n)}$ also has normalized columns and satisfies $\|\mat \Delta^{(n)} \|_F = \epsilon_n$ for $n=1, \ldots, N-1$, then $\exists \epsilon>0$ such that if $\epsilon_n < \epsilon$ for $n=1, \ldots, N-1$, then
\[\tilde{\mat{A}}^{(N)} = \mat{X}_{(N)} (\bar{\mat{A}}^{(1)}{}^{\dagger}{}^{T} \odot \cdots \odot \bar{\mat{A}}^{(N-1)}{}^{\dagger}{}^{T})\]
where $\mat { \tilde{A}}^{(N)} = \bar{\mat{A}}^{(N)}\bar{\mat{D}}$ ensures that $\bar{\mat{A}}^{(N)}$ are normalized, and satisfies
\begin{align*}
    &\|\bar{\mat{A}}^{(N)} -\mat{A}^{(N)}\|_F = O\bigg(\prod_{n=1}^{N}\epsilon_n^{N-1}\bigg), \\
    &\text{ and } \quad \|\bar{\mat{D}} - \mat D\|_F = O(\epsilon).
\end{align*}
%
\end{lemma}
\begin{proof}
Let $\epsilon < 1 $ and $\epsilon< \min_n(\sigma_{\text{min}}(\mat A^{(n)}))$ for each $n$, and therefore $\bar{\mat{A}}^{(n)}$ is full rank for each $n =1, \ldots, N-1$.
Substituting the decomposition of $\tsr{X}$ into the computed solution, we obtain
\begin{align*}
\tilde{\mat{A}}^{(N)} &= \mat{A}^{(N)} \mat D\bigg( 
 (\bar{\mat{A}}^{(1)}{}^\dagger (\bar{\mat{A}}^{(1)}-\mat{\Delta}^{(1)})) \ast \cdots \ast
 (\bar{\mat{A}}^{(N-1)}{}^\dagger (\bar{\mat{A}}^{(N-1)}-\mat{\Delta}^{(N-1)})) \bigg)^T \\
                    &= \mat{A}^{(N)} \mat D\bigg(
  (\mat{I} -\bar{\mat{A}}^{(1)}{}^\dagger \mat{\Delta}^{(1)}) \ast  \cdots \ast
  (\mat{I} -\bar{\mat{A}}^{(N-1)}{}^\dagger \mat{\Delta}^{(N-1)})\bigg)^T\\
                    &= \mat{A}^{(N)} \mat D \bigg(\mat S+ (-1)^{N-1}
  \bar{\mat{A}}^{(1)}{}^\dagger \mat{\Delta}^{(1)} \ast \cdots \ast
  \bar{\mat{A}}^{(N-1)}{}^\dagger \mat{\Delta}^{(N-1)}\bigg)^T.
\end{align*}
where $\mat{S}$ includes all cross-terms of the Hadamard products, which must be diagonal since any such term includes a Hadamard product with an identity matrix. Since,
\[ \| \mat I - \mat S\|_F = O(\max_n(\epsilon_n)) = O(\epsilon),\]
$\mat S$ is full rank for sufficiently small $\epsilon$. Let \[\mat {\Delta} =  \bigg((-1)^{N-1}
  \bar{\mat{A}}^{(1)}{}^\dagger \mat{\Delta}^{(1)} \ast \cdots \ast
  \bar{\mat{A}}^{(N-1)}{}^\dagger \mat{\Delta}^{(N-1)}\bigg)^T.\]
  Now, the norm calibration diagonal matrix is defined so that $\bar{d}_{ii} = \|\vcr{\tilde{a}}^{(N)}_i\|_2$.
  Since, 
  \[\tilde{\mat{A}}^{(N)} = \mat{A}^{(N)}\mat D(\mat S + \mat{\Delta}) = \mat A^{(n)}\mat D + \mat A^{(n)}\mat D(\mat S + \mat \Delta - \mat I),\]
  and $\| \mat S + \mat{\Delta} - \mat I\|_2=O(\epsilon)$, we have 
\[\bar{d}_{ii}=\|\vcr{\tilde{a}}^{(N)}_i\|_2 \leq \|\vcr{a}^{(N)}_i\|_2d_{ii} + \|\mat{A}^{(N)}\mat D\|_F\| \mat S + \mat{\Delta} - \mat I  \|_2 = \|\vcr{a}^{(N)}_i\|_2d_{ii} + O(\epsilon) = d_{ii} + O(\epsilon).\]
Consequently, \(\|\mat {\bar{D}} - \mat D \|_2 = O(\epsilon)\).
Further, we can obtain a tighter bound (in terms of $O(\|\mat{\Delta}\|_F)$ instead of $O(\epsilon)$) by considering, $\tilde{\mat{A}}^{(N)} = \mat{A}^{(N)}\mat D\mat S(\mat I + \mat S^{-1}\mat{\Delta})$, so
\[\bar{d}_{ii}=\|\vcr{\tilde{a}}^{(N)}_i\|_2 \leq \|\vcr{a}^{(N)}_i\|_2d_{ii}s_{ii} + \|\mat{A}^{(N)}\mat D \mat \Delta\|_F 
= d_{ii}s_{ii} + O(\|\mat{\Delta}\|_F).\]
This bound allows us to get the desired result for the error in the factor matrices,
\begin{align}
\nonumber
    \|\bar{\mat{A}}^{(N)}- \mat{A}^{(N)}\|_F =
\|\tilde{\mat{A}}^{(N)}\bar{\mat{D}}^{-1}- \mat{A}^{(N)}\|_F &= 
\|\mat{A}^{(N)}\mat D \mat S (\mat I + \mat S^{-1} \mat{\Delta})\bar{\mat{D}}^{-1}- \mat{A}^{(N)}\|_F\\
&= O(\|\mat I - \mat D \mat S (\mat I + \mat S^{-1} \mat{\Delta})\bar{\mat{D}}^{-1}\|_F).
\end{align}
Since, we have that $\|\bar{\mat D}  - \mat D \mat S\|_F = O(\|\mat{\Delta}\|_F)$,
\begin{align*}
    \|\mat I - \mat D \mat S (\mat I + \mat S^{-1} \mat{\Delta})\bar{\mat{D}}^{-1}\|_F &= \|\mat I - \mat{D}\mat S\bar{\mat D}^{-1} + \mat D \mat{\Delta}\bar{\mat{D}}^{-1}\|_F\\
&= \|\bar{\mat D} \bar{ \mat D}^{-1} - \mat{D}\mat S\bar{\mat D}^{-1} + \mat D\mat{\Delta}\bar{\mat{D}}^{-1} \|_F \\
&= O(\|\mat{\Delta}\|_F)+ \|\mat D \mat{\Delta}\bar{\mat{D}}^{-1}\|_F = O(\|\mat{\Delta}\|_F).
\end{align*}
Since, $\|\mat{\Delta}\|_F=O\Big(\prod_{n=1}^{N}\epsilon_n^{N-1}\Big),$ this completes the proof.
\end{proof}
The above Lemma states that in Algorithm~\ref{alg:mnorm_basic}, the error in the updated CPD factor relative to the true CPD factor is bounded by the product of errors in the previous $N-1$ factors. Using this error bound, we derive convergence rate for Algorithm~\ref{alg:mnorm_basic}.
\begin{lemma}
\label{lem:poly_conv}
 For any algorithm where the error in the update is of the order of product of error in previous $k$ updates, the rate of convergence is equal to the positive root of the polynomial $\alpha^{k}-\sum_{i=0}^{k-1}\alpha^i.$
\end{lemma}
\begin{proof}
Let the error at the $nth$ iteration be given as $e_n$. The error at the $nth$ iteration then satisfies the following in the worst case,
\begin{align}
\label{eq:conv_1}
    e_n = L \prod_{i=1}^{k}e_{n-i}, \quad \text{ where $L$ is a constant}.
\end{align}
The above recurrence can be solved by assuming that the error satisfies the following asymptotic relation
\begin{align}
\label{eq:conv_2}
    e_{n} = Ce_{n-1}^\alpha,
\end{align}
where $C$ is some constant and $\alpha$ is the rate of convergence. From \eqref{eq:conv_1} and \eqref{eq:conv_2},
\begin{align*}
    Ce^\alpha_{n-1} = L\prod_{i=1}^{k}e_{n-i}, \\
    \frac{C^{ \sum_{i=0}^{k-1}{\frac{1}{\alpha^i}}}}{L} = e_{n-1}^{ ( -\alpha +\sum_{i=0}^{k-1}\frac{1}{\alpha^i}) }.
\end{align*}
Since the left hand side is constant for $ n \rightarrow \infty$,
\begin{align*}
    \alpha^k - \sum_{i=0}^{k-1}\alpha^i =0.
\end{align*}
\end{proof}

Now, with all the pieces together we can show that for exact CPD, Algorithm~\ref{alg:mnorm_basic} locally converges at a rate which is given in the following theorem. 
\begin{theorem}
\label{thm:AMDM_ex_conv}
 Suppose $\tsr{X} = [\![\mat D ; \mat{A}^{(1)}, \ldots, \mat{A}^{(N)} ]\!]$, where each $\mat{A}^{(i)}\in\mathbb{R}^{s_i\times R}$ with $s_i\geq R$ is full rank and has normalized columns. Algorithm~\ref{alg:mnorm_basic} for computing the exact CP decomposition of $\tsr{X}$ converges locally with a rate of convergence equal to $\alpha^N$ where $\alpha$ is the unique real root of the polynomial $x^{N-1} -\sum_{i=0}^{N-2}x^i$.
\end{theorem}
 \begin{proof}
 Consider the computation of the CP decomposition of the tensor $\tsr{X}$ with exact rank $R$ and initial guess $\mat{ \bar{A}}^{(n)}$ with normalized columns such that $ \mat{ \bar{A}}^{(n)}  = \mat A^{(n)} +\mat  {\Delta}^{(n)}$, with $\|\mat \Delta^{(n)}\|_2 <\epsilon$ sufficiently small for $n = 1, \ldots, N-1$ (as described in Lemma~\ref{lem:err_norm_mnorm}). Let the error in CPD at $nth$ iteration be given as
 \begin{align*}
     E_n = \max \{ \|\mat{\bar{D}} - \mat D\|_F, \| \mat{\bar{A}}^{(1)} - \mat A^{(1)}\|_F, \ldots, \|\mat{\bar{A}}^{(N)} - \mat A^{(N)}\|_F \}.
 \end{align*}
  Also, let the error in $kth$ subiteration of the $nth$ iteration be given as $\epsilon_k$. From Lemma~\ref{lem:err_norm_mnorm}, we know that the error in CPD is bounded by the maximum error in factor matrices. Since the error decreases at each subiteration, $\exists n$ such that $E_n = O(\epsilon_1)$.
 
 From Lemma~\ref{lem:err_norm_mnorm}, we know that the error in a subsweep of the algorithm modulo the column scaling is of the order of product of errors in previous $N-1$ subsweeps, therefore the error at $(n+1)th$ iteration is bounded by the error in the first factor matrix, given by  \[E_{n+1}= O\Big(\prod_{k=0}^{N-2}{\epsilon_{N-k}}\Big). \]
 By using  Lemma~\ref{lem:poly_conv}, we know that the error in a subiteration is given by the following recurrence, where $\alpha$ is the positive root of $x^{N-1} - \sum_{i=0}^{N-2}x^i$, 
\begin{align*}
    \epsilon_{k+1} = O(\epsilon_k ^\alpha) \quad \text{for all k} = 1, \ldots, N.
\end{align*}
Therefore, the error at $(n+1)th$ iteration of Algorithm~\ref{alg:mnorm_basic} can be expressed as 
 \begin{align*}
     E_{n+1} &= O\Big(\prod_{i=1}^{N-1}\epsilon_1^{\alpha^i}\Big)=O( E^{\sum_{i=1}^{N-1}\alpha^i}_{n}).
 \end{align*}
Using the fact that $\alpha$ is a root of the polynomial $x^{N-1} - \sum_{i=0}^{N-2}x^i$, implies that $\alpha^{N-1} - \sum_{i=0}^{N-2}\alpha^i=0$, that is, $\sum_{i=1}^{N-1}\alpha^i = \alpha^N$.
Therefore,
 \begin{align*}
      E_{n+1} &= O\Big(E_{n}^{\alpha^N}\Big).
 \end{align*}
\end{proof}
This completes the proof to show that Algorithm~\ref{alg:mnorm_basic} locally converges superlinearly for exact CP rank cases. We verify our theoretical results in the Section~\ref{sec:exp}.

\subsection{Convergence to Other Stationary Points}
\label{subsec:conv_approx}
The result in Theorem~\ref{thm:AMDM_ex_conv} can be generalized to the case where a tensor $\tsr{X}$ can be represented as the sum of two tensors, $\tsr{T}$ and $\tsr{E}$, where $\tsr{T}$ has an underlying CPD structure of rank $R$ and $\tsr{E}$ has a CP decomposition that is mostly orthogonal to the decomposition of $\tsr{T}$.
For such an input tensor $\tsr{X}$, we show that Algorithm~\ref{alg:mnorm_basic} with CP rank $R$, locally converges to the underlying CP factors, provided that $\tsr T$ is associated with a stationary point exists.
We analyze the convergence rate of the algorithm and show that this is a generalization of the previous result, since we converge to a subset of the CP factors with same convergence rate as in Theorem~\ref{thm:AMDM_ex_conv}, if the factors of $\tsr{T}$ are in the orthogonal complement of the column space of corresponding CP factors of $\tsr{E}$.

\begin{lemma}\label{lem:err_approx_mnorm}
For a given tensor $\tsr{X}$, assume there exists a stationary point of Algorithm~\ref{alg:mnorm_basic}, yielding positive diagonal matrix $\mat D$ and factors $(\mat{A}^{(1)}, \ldots, \mat{A}^{(N)})$ where each $\mat{A}^{(i)}\in\mathbb{R}^{s_i\times R}$ with $s_i\geq R$ is full rank and has normalized columns, i.e., $\|\vcr{a}^{(i)}_j\|_2 = 1$ for all $i,j$, and their pseudoinverse-transposes $(\mat{U}^{(1)},\ldots,\mat{U}^{(N)})$, so $\mat{A}^{(n)}{}^\dagger=\mat{U}^{(n)}{}^T$.
The stationary point conditions imply that $\forall n\in\{1,\ldots, N\}$, we have
\[\mat{A}^{(n)}\mat D = \mat{X}_{(n)}\bigg(\bigodot_{m=1,m\neq n}^N \mat U^{(m)}\bigg).\]
Further, assume that $\tsr{X} =  [\![ \mat D ;\mat{A}^{(1)}, \ldots, \mat{A}^{(N)}]\!] + [\![ \hat{\mat{A}}^{(1)}, \ldots, \hat{\mat{A}}^{(N)}]\!]$ with $\|\hat{\mat{A}}^{(n)}{}^T\mat{U}^{(n)}\|_F\leq \epsilon_\perp$.
Given approximations $\tilde{\mat A}^{(n)} = \mat{A}^{(n)} + \mat{\Delta}_A^{(n)}$ and $\tilde{\mat U}^{(n)} = \tilde{\mat A}^{(n)}{}^\dagger{}^T= \mat{U}^{(n)} + \mat{\Delta}_U^{(n)}$ with normalized columns, then $\exists \epsilon>0$ such that if $\|\mat{\Delta}_A^{(n)}\|_F,\|\mat{\Delta}_U^{(n)}\|_F\leq \epsilon_n\leq \epsilon$ for $\forall n\in\{1,\ldots, N-1\}$, then
\[\tilde{\mat{A}}^{(N)} = \mat{X}_{(N)} (\tilde{\mat{U}}^{(1)} \odot \cdots \odot \tilde{\mat{U}}^{(N-1)})\]
satisfies $\|\tilde{\mat{A}}^{(N)}\bar{\mat{D}}^{-1} - \mat{A}^{(N)}\|=O(\epsilon \epsilon_\perp + \epsilon_1\ldots \epsilon_{N-1})$, where $\bar{\mat{D}}$ normalizes $\tilde{\mat{A}}^{(N)}$, i.e., $\bar{d}_{ii} = \|\tilde{\vcr{a}}^{(N)}_i\|_2$. Further, $\|\bar{\mat{D}} - \mat D\|_F = O(\epsilon).$
\end{lemma}
\begin{proof}
We expand the update as follows,
\begin{alignat*}{2}
\tilde{\mat{A}}^{(N)} &= &&([\![ \mat{A}^{(1)}, \ldots, \mat{A}^{(N)}\mat D]\!] + [\![ \hat{\mat{A}}^{(1)}, \ldots, \hat{\mat{A}}^{(N)}]\!])_{(N)} (\tilde{\mat{U}}^{(1)} \odot \cdots \odot \tilde{\mat{U}}^{(N-1)}) \\
&= &&\mat{A}^{(N)}\mat D(\mat{A}^{(1)}{}^T\tilde{\mat{U}}^{(1)} \ast \cdots \ast \mat{A}^{(N-1)}{}^T\tilde{\mat{U}}^{(N-1)})  +
   \hat{\mat{A}}^{(N)}(\hat{\mat{A}}^{(1)}{}^T\tilde{\mat{U}}^{(1)} \ast \cdots \ast \hat{\mat{A}}^{(N-1)}{}^T\tilde{\mat{U}}^{(N-1)}) \\
&= &&\mat{A}^{(N)}\mat D((\mat I + \mat{A}^{(1)}{}^T\mat{\Delta}_U^{(1)}) \ast \cdots \ast(\mat I+ \mat{A}^{(N-1)}{}^T\mat{\Delta}_U^{(N-1)})) \\
&   +&&\hat{\mat{A}}^{(N)}((\hat{\mat{A}}^{(1)}{}^T\mat{U}^{(1)} + \hat{\mat{A}}^{(1)}{}^T\mat{\Delta}_U^{(1)})\ast \cdots \ast (\hat{\mat{A}}^{(N-1)}{}^T\mat{U}^{(N-1)} + \hat{\mat{A}}^{(N-1)}{}^T\mat{\Delta}_U^{(N-1)})).
\end{alignat*}
By the stationary point condition, we have that
\[\hat{\mat{A}}^{(N)}(\hat{\mat{A}}^{(1)}{}^T\mat{U}^{(1)} \ast \cdots \ast \hat{\mat{A}}^{(N-1)}{}^T\mat{U}^{(N-1)}) = \mat 0.\]
Consequently, the error reduces to the cross terms of the summations, i.e.,\footnote{For $N=3$, the right-hand side of this formula is \begin{align*}
&\mat{A}^{(3)}\mat D[\mat{A}^{(1)}{}^T\mat{\Delta}^{(1)}_U \ast \mat{A}^{(2)}{}^T\mat{\Delta}^{(2)}_U+\mat I \ast (\mat A^{(1)}{}^T\mat{\Delta}_U^{(1)} + \mat A^{(2)}{}^T\mat{\Delta}_U^{(2)})] \\
+&\hat{\mat{A}}^{(3)}[
    \hat{\mat{A}}^{(1)}{}^T\mat{U}^{(1)} \ast \hat{\mat{A}}^{(2)}{}^T\mat{\Delta}^{(2)}_U+
    \hat{\mat{A}}^{(2)}{}^T\mat{U}^{(2)} \ast \hat{\mat{A}}^{(1)}{}^T\mat{\Delta}^{(1)}_U].
\end{align*}
}
\begin{alignat*}{2}
\tilde{\mat{A}}^{(N)} -\mat{A}^{(N)}\mat D &= &&\mat{A}^{(N)}\mat D\bigg[\bigast_{n=1}^N\mat{A}^{(n)}{}^T\mat{\Delta}_U^{(n)}+\mat I\ast \sum_{k=1}^{N-1}\bigg(\sum_{\{m_1,\ldots,m_k\}\subset\{1,\ldots, N\}} \bigast_{l=1}^k \mat{A}^{(m_l)}{}^T\mat{\Delta}_U^{(m_l)}\bigg)\bigg] \\
&   +&&\hat{\mat{A}}^{(N)}\bigg[\sum_{k=1}^{N-1}\bigg(\sum_{\substack{\{m_1,\ldots,m_k\}\subset\{1,\ldots, N\},\\ \{w_1,\ldots,w_{N-k}\}=\{1,\ldots,N\}\setminus\{m_1,\ldots,m_k\}}} \bigast_{l=1}^k \hat{\mat{A}}^{(m_l)}{}^T\mat{U}^{(m_l)}\bigast_{l=1}^{N-k}\hat{\mat{A}}^{(w_l)}{}^T\mat{\Delta}_U^{(w_l)}\bigg)\bigg]
\end{alignat*}
First, since each error term has Frobenius norm $O(\|\mat{\Delta}^{(n)}\|)=O(\epsilon)$, the column norms of $\tilde{\mat{A}}^{(N)}$ will yield $\bar{\mat D}$ with \(\|\mat{\bar{D}} - \mat D \|_F = O(\epsilon)\).
Further, in order to get the error bound on $\tilde{\mat{A}}^{(N)}$, we consider the diagonal matrix,
\[\mat S = \mat I + \mat I\ast \sum_{k=1}^{N-1}\bigg(\sum_{\{m_1,\ldots,m_k\}\subset\{1,\ldots, N\}} \bigast_{l=1}^k \mat{A}^{(m_l)}{}^T\mat{\Delta}_U^{(m_l)}\bigg),\]
and show that \(\|\mat{\bar{D}} - \mat D\mat S \|_F = O(\epsilon_1\ldots \epsilon_{N-1}+\epsilon\epsilon_\perp)\).
Since,
\begin{alignat*}{2}
\tilde{\mat{A}}^{(N)} -\mat{A}^{(N)}\mat D\mat S &= &&\mat{A}^{(N)}\mat D\mat S\bigg[\mat S^{-1}\bigast_{n=1}^N\mat{A}^{(n)}{}^T\mat{\Delta}_U^{(n)}\bigg]\\
&   +&&\hat{\mat{A}}^{(N)}\bigg[\sum_{k=1}^{N-1}\bigg(\sum_{\substack{\{m_1,\ldots,m_k\}\subset\{1,\ldots, N\},\\ \{w_1,\ldots,w_{N-k}\}=\{1,\ldots,N\}\setminus\{m_1,\ldots,m_k\}}} \bigast_{l=1}^k \hat{\mat{A}}^{(m_l)}{}^T\mat{U}^{(m_l)}\bigast_{l=1}^{N-k}\hat{\mat{A}}^{(w_l)}{}^T\mat{\Delta}_U^{(w_l)}\bigg)\bigg], \end{alignat*}
and $\|\mat S-\mat I\|_F=O(\epsilon)$, we have
\[\|\tilde{\mat{A}}^{(N)} -\mat{A}^{(N)}\mat D\mat S\|_F = O(\epsilon_1\ldots\epsilon_{N-1} + \epsilon \epsilon_\perp).\]
The same bound follows for each column, so \(\|\mat{\bar{D}} - \mat D\mat S \|_F =O(\epsilon_1\ldots\epsilon_{N-1} + \epsilon \epsilon_\perp)\).
Now, using this bound on
\begin{alignat*}{2}
\tilde{\mat{A}}^{(N)}\bar{\mat D}^{-1} &= &&\mat{A}^{(N)}\mat D\mat S\bigg[\mat S^{-1}\bigast_{n=1}^N\mat{A}^{(n)}{}^T\mat{\Delta}_U^{(n)}+\mat I\bigg]\bar{\mat D} ^{-1}\\
&   + &&\hat{\mat{A}}^{(N)}\bigg[\sum_{k=1}^{N-1}\bigg(\sum_{\substack{\{m_1,\ldots,m_k\}\subset\{1,\ldots, N\},\\ \{w_1,\ldots,w_{N-k}\}=\{1,\ldots,N\}\setminus\{m_1,\ldots,m_k\}}} \bigast_{l=1}^k \hat{\mat{A}}^{(m_l)}{}^T\mat{U}^{(m_l)}\bigast_{l=1}^{N-k}\hat{\mat{A}}^{(w_l)}{}^T\mat{\Delta}_U^{(w_l)}\bar{\mat D}^{-1}\bigg)\bigg],
\end{alignat*}
we obtain
\begin{alignat*}{2}
\|\tilde{\mat{A}}^{(N)}\bar{\mat D}^{-1} -\mat A^{(N)}\|_F&= \Bigg\|\mat{A}^{(N)}\mat D\bigg[\bigast_{n=1}^N\mat{A}^{(n)}{}^T\mat{\Delta}_U^{(n)}\bigg]\bar{\mat D}^{-1} \\
& + \hat{\mat{A}}^{(N)}\bigg[\sum_{k=1}^{N-1}\bigg(\sum_{\substack{\{m_1,\ldots,m_k\}\subset\{1,\ldots, N\},\\ \{w_1,\ldots,w_{N-k}\}=\{1,\ldots,N\}\setminus\{m_1,\ldots,m_k\}}} \bigast_{l=1}^k \hat{\mat{A}}^{(m_l)}{}^T\mat{U}^{(m_l)}\bigast_{l=1}^{N-k}\hat{\mat{A}}^{(w_l)}{}^T\mat{\Delta}_U^{(w_l)}\bar{\mat D}^{-1}\bigg)\bigg]\bigg \|_F \\
 &+O(\epsilon_1\ldots\epsilon_{N-1} + \epsilon \epsilon_\perp) \\
&= O(\epsilon_1\ldots\epsilon_{N-1} + \epsilon \epsilon_\perp).
\end{alignat*}
%
%
%
\end{proof}

\section{Approximate Decomposition}
\label{sec:mnorm}
In the above sections, we have provided a motivation for Algorithm~\ref{alg:mnorm_basic} to compute a CP decomposition of rank $R$ with $R$ being less than or equal to the smallest mode length of the input tensor. We have shown in Theorem~\ref{thm:AMDM_ex_conv} that this algorithm exhibits a super linear local convergence rate for exact CP decomposition problems and achieves a desirable approximation for special input tensors as described in Lemma~\ref{lem:err_approx_mnorm}.  We now focus on the case of finding a good CP approximation for an arbitrary input tensor. We show that Algorithm~\ref{alg:mnorm_basic} can be viewed as performing coupled minimization of the residual error of the decomposition in terms of a Mahalanobis distance metric~\cite{chandra1936generalised}.
Note that this perspective of the algorithm is different from the one introduced in Section~\ref{subsec:spectral_lagrange}, however it allows us to formulate an alternating minimization algorithm which generalizes the Algorithm~\ref{alg:mnorm_basic} to any CP rank and to interpolate between the updates of ALS and Algorithm~\ref{alg:mnorm_basic}.

\subsection{Mahalanobis Distance Minimization}

Each update of Algorithm~\ref{alg:mnorm_basic} may be viewed as minimizing a residual error with rescaled components.
For an order 3 tensor $\tsr{X}$ in updating the first factor, it minimizes
\begin{align}
\label{eq:transformed_obj}
    \|(\tsr{X} - [\![ \mat{A}, \mat{B}, \mat{C} ]\!])_{(1)}(\mat{C}^\dagger{}^T \otimes \mat{B}^\dagger{}^T)\|_F.
\end{align}
Since the column span of $\mat{C}\odot \mat{B}$ is the same as that of $\mat{C}^\dagger{}^T \otimes \mat{B}^\dagger{}^T$, the transformed residual preserves all components of the residual error that may be reduced in choosing $\mat{A}$ with $\mat{B}$ and $\mat{C}$ fixed, since the residual may be written as
\[\mat A (\mat I_R \odot \mat I_R)^T - \mat{X}_{(1)}(\mat C^{\dagger T}\otimes \mat B^{\dagger T}).\]
We show that this may be viewed as optimizing a single overall objective function relative to each factor, while keeping the distance metric associated with that factor independent (and then updating it thereafter).
This interpretation then enables us to extend Algorithm~\ref{alg:mnorm_basic} for any CP rank and introduce methods that are a hybrid of Algorithm~\ref{alg:mnorm_basic} and standard ALS.



\subsubsection{Alternating Mahalanobis Distance Minimization}
\label{subsec:Alternating_mahalanobis}
We consider a variant of Mahalanobis distance~\cite{de2000mahalanobis},  which computes the distance between vectors $\vcr x$ and $\vcr y$ as $d(\vcr x,\vcr y)=(\vcr x-\vcr y)^T\mat M(\vcr x - \vcr y)$ for a given symmetric positive definite matrix $\mat{M}$. The matrix $\mat{M}$ is called the ground metric matrix. Ground metric generalizes the Euclidean distance to Mahalanobis distance by rotation and scaling of the axes along which the distance is computed. While, the underlying ground metric may already be known, it may also be learned via various metric learning techniques~\cite{bellet2013survey, kulis2013metric}.
In optimal transport applications and other applications which require computation of distance between probability distributions, a Wasserstein distance is considered instead. Wasserstein distance between tensors with a given ground metric has been considered for nonnegative CP decomposition~\cite{afshar2021swift}.  The ground metric in Wasserstein distance may be learned via similar metric learning techniques as in Mahalanobis distance~\cite{cuturi2014ground}. Simultaneous optimization for a ground metric and Wasserstein distance between matrices has been used for nonnegative matrix factorization~\cite{zen2014simultaneous}. 

We consider minimization of the Mahalanobis distance between tensors with a fixed ground metric which maybe updated later.
In particular, the objective function minimized for an input tensor $\tsr{X} \in \mathbb{R}^{I_1 \dots I_N}$, and factors $\mat A^{(i)} \in \mathbb{R}^{I_i \times R}$ is
\begin{align}
    \label{eq:obj_mnorm}
    f(\mat{A}^{(1)}, \cdots , \mat{A}^{(N)} ) &= \frac{1}{2} \text{vec}\big(\tsr{X} - \tsr{Y} \big)^T\mat{M}\text{vec}\big(\tsr{X} -  \tsr{Y}\big),\\
    \notag
    \text{where } \tsr{Y}&=[\![ \mat{A}^{(1)}, \cdots , \mat{A}^{(N)} ]\!].
\end{align}
We restrict the ground metric matrix $\mat M$ to be Kronecker structured defined as
\begin{align*}
    \mat M &= \bigotimes_{k=1}^N\mat M^{(k)-1}.
\end{align*}
Each $\mat{M}^{(k)-1}$ maybe viewed as a ground metric for each mode of the tensor. This restriction allows us to exploit the computational benefits of the structure and enables us to formulate an efficient alternating minimization algorithm. We consider the objective in \eqref{eq:obj_mnorm} for a general (fixed) ground metric for alternating optimization, which also allows us to formulate different algorithms for CP decomposition by changing the ground metric.
We derive an update for the alternating minimization with respect to $n$th factor matrix, given by
\begin{align}
\label{eq:obj_amdm}
    \mat{A}^{(n)} &= \min_{\mat A^{(n)}}\frac{1}{2}\|\mat X_{(n)} - \mat A^{(n)}\mat P^{(n)T}\|^2_{\mat M}, \\
    \nonumber
    &\text{where } \mat P^{(n)} = \bigodot_{m=1,m\neq n}^N\mat{A}^{(m)}.
\end{align} 
For succinct writing, let $\mat M_{(n)} =\bigotimes_{k=1, k\neq n}^N\mat M^{(k)-1}$. Since the objective function is quadratic in $\mat A^{(n)}$,
a minimizer of (\ref{eq:obj_amdm}) can be found by obtaining obtaining a gradient $\mat G^{(n)}$ with respect to the $n$th factor matrix and setting it to $\mat{0}$. The gradient is 
\[\mat G^{(n)} =\mat M^{(n)-1}\mat A^{(n)}\mat{P}^{(n)T}\mat{M}_{(n)}\mat{P}^{(n)} - \mat M^{(n)-1}\mat X_{(n)}\mat{M}_{(n)}\mat{P}^{(n)}.
\]
Setting the gradient above to be $\mat 0$ and equating $\mat M^{(n)}\mat M^{(n)-1} = \mat I$, we get an update for the $nth$ factor given as the solution of the following system,
\begin{align}
\label{eq: gen_update_amdm}
    \mat A^{(n)}\Big(\mat{P}^{(n)T}\mat{M}_{(n)}\mat{P}^{(n)}\Big) =\mat X_{(n)}\mat{M}_{(n)}\mat{P}^{(n)}.
\end{align}
Using the properties of Khatri-Rao products and Kronecker products, the update for the $n$th factor matrix reduces to the system of equations
\begin{align}
\label{eq:update_gen}
    \mat A^{(n)}\mat Z^{(n)} &=\mat X_{(n)}\mat{L}^{(n)},\\
    \notag
    \text{where }\mat L^{(n)} &= \bigodot_{k=1, k \neq n}^N \mat M^{(k)-1}\mat A^{(k)}, \\
    \notag
    \text{and } \mat Z^{(n)} &= \bigast_{k=1, k \neq n}^N\mat A^{(k)T}\mat M^{(k)-1} \mat A^{(k)}.
\end{align}
The above update leads to the ALS algorithm if $\mat M^{(k)} = \mat I$ for all $k$. We can retrieve Algorithm~\ref{alg:mnorm_basic} by defining 
\begin{align}
\label{eq:M_in_Mnorm}
    \mat M^{(k)} =\mat A^{(k)}\mat A^{(k)T} + (\mat I - \mat{A}^{(k)}\mat{A}^{(k)}{}^\dagger), \quad  \forall k \in \{1,\ldots,N\}.
\end{align}
The matrix $\mat I - \mat{A}^{(k)}\mat{A}^{(k)}{}^\dagger$ is inconsequential when applied to the factor matrices. It is included to ensure that each $\mat M^{(k)}$ is SPD. Since Algorithm~\ref{alg:mnorm_basic} can be retrieved from the above update, we refer to Algorithm~\ref{alg:mnorm_basic} as AMDM (Alternating Mahalanobis Distance Minimization).

Let us assume that the iteration involving the above derived alternating updates to each factor as in ~\eqref{eq: gen_update_amdm}, converges to a critical point. We can then bound the backward error in application of each matricization of the reconstructed tensor $\tsr{Y}=[\![ \mat{A}^{(1)}, \cdots , \mat{A}^{(N)} ]\!]$, since from~\eqref{eq: gen_update_amdm}, for each $n$, we have
\begin{align*}
    \mat A^{(n)}\Big(\mat{P}^{(n)T}\mat{M}_{(n)}\mat{P}^{(n)}\Big) -\mat X_{(n)}\mat{M}_{(n)}\mat{P}^{(n)} &= \mat 0, \\
\Big (\mat{Y}_{(n)} - \mat{X}_{(n)} \Big) \mat M_{(n)} \mat P^{(n)} &= \mat 0.
\end{align*}
Therefore, we have that
\begin{align*}
    \|\mat{Y}_{(n)}\vcr z - \mat{X}_{(n)} \vcr z \| = \| \mat X_{(n)} \vcr z^\perp \|,
\end{align*} where $\vcr z^\perp$ is the projection of $\vcr z \in \mathbb{R}^{\prod_{j=1 j \neq n} I_j}$ onto the orthogonal complement of column span of $\mat M_{(n)} \mat P^{(n)}$ or $\bigodot_{j=1, j\neq n}^N\mat M^{(j)} \mat A^{(j)}$. For ALS, AMDM, and the hybrid methods (introduced in Section~\ref{subsec:hybrid}) that interpolate between the both, it is sufficent to consider the projection onto the orthogonal complement of column span of $\bigodot_{j=1, j\neq n}^N \mat A^{(j)}$. This is because for each $j$, the ground metric matrices are chosen such that the column span of $\mat A^{(j)}$ is an invariant subspace of $\mat M^{(j)}$. 

\subsubsection{Comparison of AMDM and ALS for approximate rank-2 CPD}
We use the formulation introduced in the previous subsection to generalize AMDM to the case when CP rank $R$ is greater than the mode lengths and to derive hybrid methods that interpolate between AMDM and ALS.
The residual transformation tends to equalize the weight of contribution to the objective function attributed to components of the error associated with different rank-1 parts of the CP decomposition, $[\![ \vcr{a}_i, \vcr{b}_i, \vcr{c}_i ]\!]$ without increasing the collinearity of columns of the factors. We provide an example as an intuition for this assertion.

Consider a tensor $\tsr{X} = \lambda_1\tsr{X}_1 + \lambda_2\tsr{X}_2 + \tsr{N}$ where $\tsr{X}_1$ and $\tsr{X}_2$ are normalized rank-$1$ tensors and $\tsr{N}$ is noise of small magnitude. Assume that $ \lambda_1\gg \lambda_2$ and the rank-$1$ tensors are highly correlated, i.e., the factors have collinear columns. Let the current CPD approximation be $\tsr{Y} = [\![ \vcr{ \bar{\lambda}} ; \mat{A}, \mat{B}, \mat{C} ]\!] =  \bar{\lambda}_1\tsr{Y}_1 + \bar{\lambda}_2 \tsr{Y}_2$. 
The least squares objective minimizes
\begin{align*}
    \|\tsr{X} - \tsr{Y}\|_F^2 &=\text{vec} (\tsr{X} - \tsr{Y})^T \text{vec} (\tsr{X} - \tsr{Y}) \\
    &= \text{vec}(\tsr{E}_1 + \tsr{E}_2 + \tsr{N})^T \text{vec}(\tsr{E}_1 + \tsr{E}_2 + \tsr{N})\\
    &= \|\tsr{E}_1\|_F^2 + \|\tsr{E}_2\|_F^2 + 2\text{vec}(\tsr{E}_1)^T \text{vec}(\tsr{E}_2) + 2\text{vec}(\tsr{N})^T(\tsr{E}_1 + \tsr{E}_2 + \tsr{N}),
\end{align*}
where $\tsr{E}_1 = \lambda_1\bar{\tsr{X}}_1 - \bar{\lambda}_1 \tsr{Y}_1$ and $\tsr{E}_2 = \lambda_2\tsr{X}_2 - \bar{\lambda}_2 \tsr{Y}_2$. Alternating least squares algorithm may reduce $\|\tsr{E}_1\|_F$ and the component of $\tsr{E}_2$ in the direction of $\tsr{E}_1$, since it leads to reduction of the terms with larger contribution in the error. This causes an increase in collinearity of the approximated factors and a more ill-conditioned decomposition. On the other hand, the objective in~\eqref{eq:transformed_obj} can be expressed as 
\begin{align*}
    \|(\tsr{X} - \tsr{Y}) \times_1 \mat I \times_2 \mat B^{\dagger}{}^T \times_3 \mat C^{\dagger}{}^T   \|_F^2 &=\text{vec} (\tsr{X} - \tsr{Y})^T \mat M \text{vec} (\tsr{X} - \tsr{Y})\\
    &= \|\tsr{E}_1\|_{\mat M}^2 + \|\tsr{E}_2\|_{\mat M}^2 + 2\text{vec}(\tsr{E}_1)^T\mat M \text{vec}(\tsr{E}_2) + 2\text{vec}(\tsr{N})^T\mat M(\tsr{E}_1 + \tsr{E}_2 + \tsr{N}),
\end{align*}
where $\mat M = \mat C^{\dagger}{}^T \mat C^\dagger  \otimes \mat B^{\dagger}{}^T \mat B^\dagger  \otimes \mat I$. The matrix $\mat M$ rescales the components of the error according to the inverse of square of singular values of the factors, since
\begin{align*}
    \|\tsr{E}_1\|^2_{\mat M} &= \text{vec}(\tsr{E}_1)^T \mat M \text{vec}(\tsr{E}_1) \\
    &= \text{vec}(\bar{\tsr{E}}_1)^T (\mat \Sigma^{-2}_{\mat C} \otimes  \mat \Sigma^{-2}_{\mat B} \otimes \mat I ) \text{vec}(\bar{\tsr{E}}_1),
\end{align*}
where $\bar{\tsr{E}}_1 = \tsr{E}_1 \times_1 \mat I \times_2 \mat U_{\mat B} \times_3 \mat U_{\mat C}$, with $\mat U_{\mat B}$ and $\mat U_{\mat C}$ being the left singular vectors of $\mat B$ and $\mat C$ respectively. Thus, the error is rotated by the left singular vectors of $\mat B$ and $\mat C$, and then rescaled by square of inverse of singular values of $\mat B$ and $\mat C$, i.e., $\Sigma^{-2}_{\mat B}$, $\Sigma^{-2}_{\mat C}$.  Therefore, if the approximated factors are collinear, the contribution in the direction of the singular vectors of $\mat B$ and $\mat C$ with largest singular value is weighed proportionally less and similarly the contribution of error in the direction of those with smaller singular value is weighed more. This reduces the imbalance in error and leads to a better conditioned decomposition.

\subsection{Generalizing AMDM to Any CP Rank}
The AMDM algorithm as described in Algorithm~\ref{alg:mnorm_basic} imposes a constraint that CP rank should be less than or equal to the smallest mode length of the tensor. We now describe how the update in~\eqref{eq:update_gen} with the ground metric as defined in~\eqref{eq:M_in_Mnorm} leads to the definition of AMDM without conditions imposed on the rank. For each $\mat M^{(k)}$, 
\begin{align*}
    \mat M^{(k)-1} = \mat A^{(k)\dagger T} \mat A^{(k)\dagger} + (\mat I - \mat A^{(k)} \mat A^{(k)\dagger}).
\end{align*}
The linear system for updating the $n$th factor matrix as in (\ref{eq:update_gen}) can then be simplified to get
\begin{align*}
    \begin{split}
        \mat A^{(n)}\mat Z^{(n)} =\mat X_{(n)}\mat{L}^{(n)}&,\\
        \text{where }\mat L^{(n)} = \bigodot_{k=1, k \neq n}^N \mat M^{(k)-1}\mat A^{(k)} &= \bigodot_{k=1, k \neq n}^N\mat A^{(k)\dagger}{}^T, \\
        \text{and } \mat Z^{(n)} = \bigast_{k=1, k \neq n}^N\mat A^{(k)T}\mat M^{(k)-1} \mat A^{(k)} &= \bigast_{k=1, k \neq n}^N\mat A^{(k)\dagger}\mat A^{(k)}.
    \end{split}
\end{align*}
The above update is equivalent to Algorithm~\ref{alg:mnorm_basic} when CP rank $R \leq I_m$, $\forall m \in \{1,\dots,N\}$, since in that  case $\mat{Z}^{(n)} = \mat I$ for all $n$. For the case when CP rank $R$ is larger than the mode lengths, we get a symmetric semi-definite system of equations. Since the system is semi-definite, a pivoted Cholesky decomposition followed by a triangular solve be used for the solution to exist. Alternatively, as in ALS, a regularization term may be introduced to make the system positive definite.
The cost of forming this system, $\mat Z^{(n)}$ is of the same leading order as ALS, i.e., $O(IR^2)$ per subsweep, since it requires to obtain a pseudo inverse of the factor in $(n-1)$th iteration, matrix multiplication and Hadamard products of the factors and their previously obtained pseudoinverses. The system solve amounts to a computational cost of $O(R^3)$ to solve the system.

\subsection{Interpolating Between AMDM and ALS}
\label{subsec:hybrid}
Performing alternating Mahalanobis distance minimization with an identity ground metric is equivalent to performing ALS. To explore methods that interpolate between AMDM and ALS, we can interpolate the ground metric between the identity matrix and the one associated with AMDM given in~\eqref{eq:M_in_Mnorm}. We can find such ground metrics by decomposing each factor matrix into two low rank matrices, such that $\mat A^{(k)} = \mat A_1^{(k)} + \mat A^{(k)}_2$, where $\mat A^{(k)}_1$ is the best rank-$t$ approximation of $\mat A^{(k)}$. In other words, $\mat A^{(k)}_1$ contains the largest $t$ singular values and the corresponding  singular vectors of $\mat A^{(k)}$ and $\mat A^{(2)}$ contains the rest. The ground metric for each mode can then be defined using only the first part as
\begin{align}
\label{eq:M_in_gen_AMDM}
    \mat M^{(k)} =\mat A^{(k)}_1\mat A^{(k)T}_1 + (\mat I - \mat{A}^{(k)}_1\mat{A}_1^{(k)}{}^\dagger), \quad  \forall k \in \{1,\ldots,N\}.
\end{align}
By defining a ground metric based on only the first part of the singular value decomposition of the factors leads to hybrid methods, since if the first part is all of the singular value decomposition then we get back the ground metric in AMDM and if it is none of the same then we get back the identity matrix by convention as the orthogonal complement of none is everything.  

Note that for each $k$,
\begin{align*}
    \mat M^{(k)-1} = \mat A_1^{(k)\dagger T} \mat A_1^{(k)\dagger} + (\mat I - \mat A_1^{(k)} \mat A_1^{(k)\dagger}).
\end{align*}
The update for the $n$th factor matrix also becomes a combination of AMDM and ALS where the first part of the singular value decomposition of factors is treated as in AMDM and the second one as in ALS. More precisely, the system of equations for updating the $n$th factor matrix is
\begin{align}
    \notag
    \mat A^{(n)}\mat Z^{(n)} =\mat X_{(n)}\mat{L}^{(n)}&,\\
    \notag
    \text{where }\mat L^{(n)} = \bigodot_{k=1, k \neq n}^N \mat M^{(k)-1}\mat A^{(k)} &= \bigodot_{k=1, k \neq n}^N(\mat A_1^{(k)\dagger} + \mat A_2^{(k)}), \\
    \text{and } \mat Z^{(n)} = \bigast_{k=1, k \neq n}^N\mat A^{(k)T}\mat M^{(k)-1} \mat A^{(k)} 
    &= \bigast_{k=1, k \neq n}^N(\mat A^{(k)\dagger}_1 \mat A^{(k)}_1 + \mat A_2^{(k)T} \mat A^{(k)}_2).
    \label{eq:update_amdm}
\end{align}
We describe the above derived hybrid algorithm in Algorithm~\ref{alg:mnorm}.
The algorithm starts by normalizing columns of all the factors and absorbing norms in the first factor as described in Section~\ref{sec:alg} and then computing a reduced SVD of all the factors which costs $O( \sum_{n=1}^NI_nR \min(I_n,R))$. 
 At the $n$th subsweep of the algorithm, operations performed are similar in computational cost as that of ALS. Right and left hand sides of the system, $\mat L_n$ and $\mat  Z_n$ require $O(I_nR^2)$ and $O(R^2 \min (I_n,R))$ operations respectively. The symmetric semi definite system solve requires $O(R^3)$. Computing the right hand side, i.e., performing MTTKRP is the most computationally expensive operation with a cost of $O(\prod_{n=1}^N I_nR)$ for each subsweep. In addition to this, the factor matrix obtained after the solve is normalized and a reduced SVD is obtained to update the singular value decomposition which costs $O(I_nR\min(I_n,R))$. Therefore, the asymptotic computational cost of the algorithm is the same as ALS being $O(\prod_{n=1}^N I_nR)$.

\begin{algorithm}[]

    \caption{\textbf{General-AMDM}: Alternating Mahalanobis Distance Minimization with singular value thresholding}
\begin{algorithmic}[1]
\small
\State{\textbf{Input: }Tensor $\tsr{X}\in\mathbb{R}^{I_1\times\cdots\times I_N}$, 
threshold $t$, rank $R$}
\State{Initialize $\{\mat{A}^{(1)}, \ldots , \mat{A}^{(N)}\}$ so each $\mat{A}^{(n)}\in\mathbb{R}^{I_n\times R}$ is random
}

\For{\texttt{$n \in \inti{2}{N} $}}
    \State{$\mat A^{(n)} =$ normalize($\mat A^{(n)})$}
    \State{$\mat U^{(n)} =$ $\min{(I_n,R)}$ left singular vectors of $\mat A^{(n)}$}
    \State{$\mat V^{(n)} =$ $\min{(I_n,R)}$ right singular vectors of $\mat A^{(n)}$}
    \State{$\vcr{s}^{(n)} =$ $\min{(I_n,R)}$ singular values of $\mat A^{(n)}$}
\EndFor
\While{\texttt{Convergence}}
    \For{\texttt{$n \in \inti{1}{N} $}}
        \For{\texttt{$m \in \inti{1}{N},m \neq n$}}
            \State{$\vcr{s}^{(m)}_\text{ps}=$ first $t$ values of $\vcr s^{(m)}$ inverted and others as it is}
            \State{$\mat L_m = \mat U^{(m)}\textbf{diag}( \vcr s^{(m)}_\text{ps})\mat V^{(m)T}$}
            \State{$\mat Z_m = \mat V^{(m)}\textbf{diag}(\vcr s^{(m)}_\text{ps}\ast \vcr s^{(m)})\mat V^{(m)T}$}
        \EndFor
    \State{Solve for $\mat A^{(n)}$ in $\mat A^{(n)}\mat Z^{(n)} = \mat X_{(n)} \mat L^{(n)} $ \quad as in \eqref{eq:update_gen}}
    \State{\texttt{Check Convergence, if converged: Break}}
    \State{$\mat A^{(n)} =$ normalize$(\mat A^{(n)})$}
    \State{Update $\mat U^{(n)}=$ $\min{(I_n,R)}$ left singular vectors of $\mat A^{(n)}$}
    \State{Update $\mat V^{(n)}=$ $\min{(I_n,R)}$ right singular vectors of $\mat A^{(n)}$}
    \State{Update $\vcr s^{(n)}=$ $\min{(I_n,R)}$ singular values of $\mat A^{(n)}$}
    \EndFor
\EndWhile
\State{\Return factor matrices $\{\mat{A}^{(1)}, \ldots , \mat{A}^{(N)}\}$ }
\end{algorithmic}
\label{alg:mnorm}
\end{algorithm}


\section{Numerical Experiments}
\label{sec:exp}
We perform numerical experiments to demonstrate the convergence behaviour of the AMDM algorithm for various tensors which include synthetic examples and tensors arising in different applications. We use absolute residual and fitness of the decomposition in Frobenius norm as metrics to measure the closeness of the decomposition to the input tensor. For an input tensor $\tsr{X}$, these are given as
\[ r= \|\tsr{X} - \tsr{Y} \|_F \text{ and } f= 1 - \frac{\|\tsr{X} - \tsr{Y} \|_F }{\|\tsr{X}\|_F, }
\]
respectively, where $\tsr{Y} =[\![ \mat{A}^{(1)}, \cdots , \mat{A}^{(N)} ]\!]$ is the approximated tensor. For measuring the stability of the decomposition or the degree of overlap of CP components, we use the normalized CPD condition number~\cite{breiding2018condition} to measure the sensitivity or degree of overlap of rank $1$ components of the decomposition. 

The normalized CP condition number is given by the reciprocal of the smallest singular value of the Terracini's matrix associated with the CP decomposition. For an equidimensional tensor of order $N$ with mode length $s$ and CP rank $R$, the size of Terracini's matrix is $s^N \times (N(s-1) + 1)R$. For CP rank lower than mode lengths of the tensor, this matrix can be compressed to $R^N \times (N(R-1) + 1)R$ and the CPD condition number can be efficiently computed with a cost of $O(R^{N+4})$. The details of  computation of the condition number are in Appendix~\ref{sec:app}. Our experiments consider two types of synthetic tensors.

\noindent
\textbf{Tensor made by random matrices} (\textit{Random} tensor). We create these tensors based on known uniformly distributed randomly-generated factor matrices $\mat{A}^{(n)}\in (0,1)^{s\times R}$, 
   \(
    \tsr{X} = [\![ \mat{A}^{(1)}, \ldots , \mat{A}^{(N)} ]\!]. 
    \)

\noindent
\textbf{Tensor made by collinear random matrices} (\textit{Collinearity} tensor). We use a similar approach as used in~\cite{acar2011scalable} to generate factor matrices with a fixed value of collinearity, say C. That means that these tensors are created with randomly-generated factors $\mathbf{A}^{(n)} \in \mathbb{R}^{s \times R}$ with the following property,
\begin{align*}
    \mathbf{a}_r^{(n)T}\mathbf{a}_z^{(n)} &=C, \\
    \text{s.t. } \|\mathbf{a}^{(n)}_r\|&=1, \forall r \neq z \in \{1,\cdots,R\}.
\end{align*}
We then set $\lambda_i = i$, $\forall i \in \{1,\cdots,R\}$ to create a tensor,
   \(
    \tsr{X} = [\![ \boldsymbol{\lambda} ; \mat{A}^{(1)}, \ldots , \mat{A}^{(N)} ]\!]. 
    \)

We consider four tensors from various real world applications.

\noindent
\textbf{Sleep-EDF tensor}: This dataset has been used to identify sleeping patterns. It comprises of electroencephalogram (EEG), electromyography (EMG)  data  in the non-rapid eye movement (NREM) stage of sleep~\cite{kemp2000analysis}.

\noindent
\textbf{MGH tensor}: This dataset consists of data from  Massachusetts General Hospital. It includes combinations of electroencephalogram (EEG), respiratory signals, and electromyogram signals (EMG). This dataset was used to analyze sleep using deep neural networks~\cite{biswal2018expert}.

\noindent
\textbf{SCF tensor}.
    We consider the density fitting tensor (Cholesky factor of the two-electron integral tensor) arising in quantum chemistry.
    This tensor has been used previously in~\cite{singh2021comparison} to compare the efficacy of the Gauss-Newton and alternating least squares algorithm.
    We leverage the PySCF library~\cite{sun2018pyscf} to generate the three dimensional compressed density fitting tensor, representing the compressed restricted Hartree-Fock wave function of water molecule chain systems with STO-3G basis set. The number of molecules in the system is set to three for this experiment.

\noindent
\textbf{Amino acid tensor}. This data set consists of five simple laboratory-made samples. 
Each sample contains different amounts of tyrosine, tryptophan and phenylalanine 
dissolved in phosphate buffered water. 
The samples were measured by fluorescence~\cite{bro1997parafac}.

The experiments are divided broadly into two categories,

\noindent
\textbf{Exact decomposition.} We create synthetic tensors with known CP rank $R$ and compare the convergence behaviour of the AMDM algorithm with the alternating least squares algorithm for exact CP decomposition.

\noindent
\textbf{Approximate decomposition.}
We create synthetic tensors with known CP rank and special structure such as with added noise or as described in Lemma~\ref{lem:err_approx_mnorm}. We then approximate these tensors with CP rank $R$ which is lower than the underlying decomposition rank.  We also consider real world tensors from different applications with unknown CP rank.

\subsection{Exact CP decomposition}
We compare alternating least squares and Algorithm~\ref{alg:mnorm_basic} for computing exact CP decomposition of synthetic tensors in Figure~\ref{fig:higher_order_conv} and~\ref{fig:roc_and_first_ord_conv} and verify our theoretical results. We create \textit{Collinearity} and \textit{Random} tensors of specified CP rank to analyze the convergence of these algorithms.

\begin{figure*}[t]
\centering
\begin{subfigure}[Random tensor residual] {\label{fig:rand_ord}\includegraphics[width=0.45\textwidth, keepaspectratio]{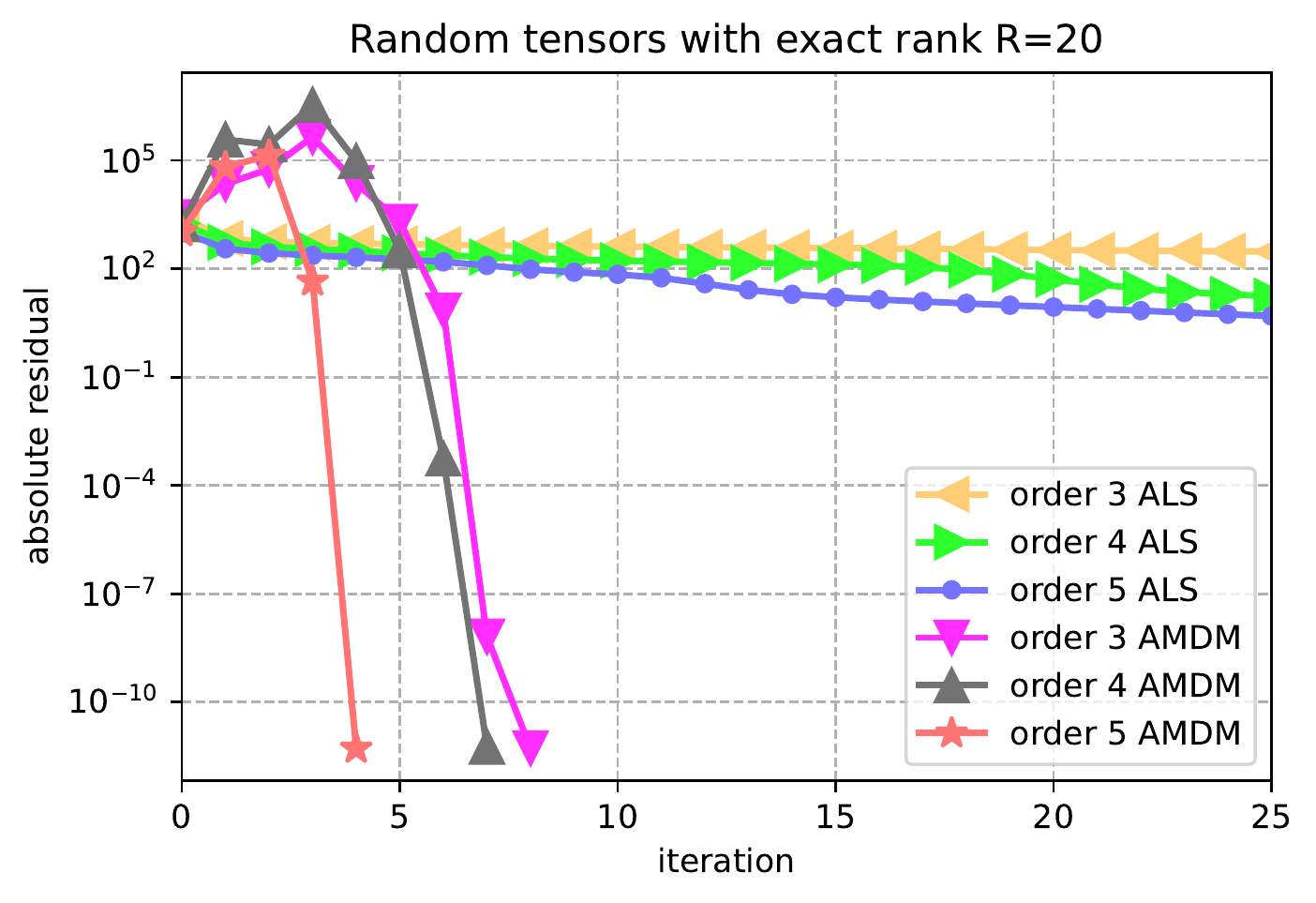}} \end{subfigure}
\begin{subfigure}[Collinearity tensor residual] {\label{fig:coll_ord}\includegraphics[width=0.45\textwidth, keepaspectratio]{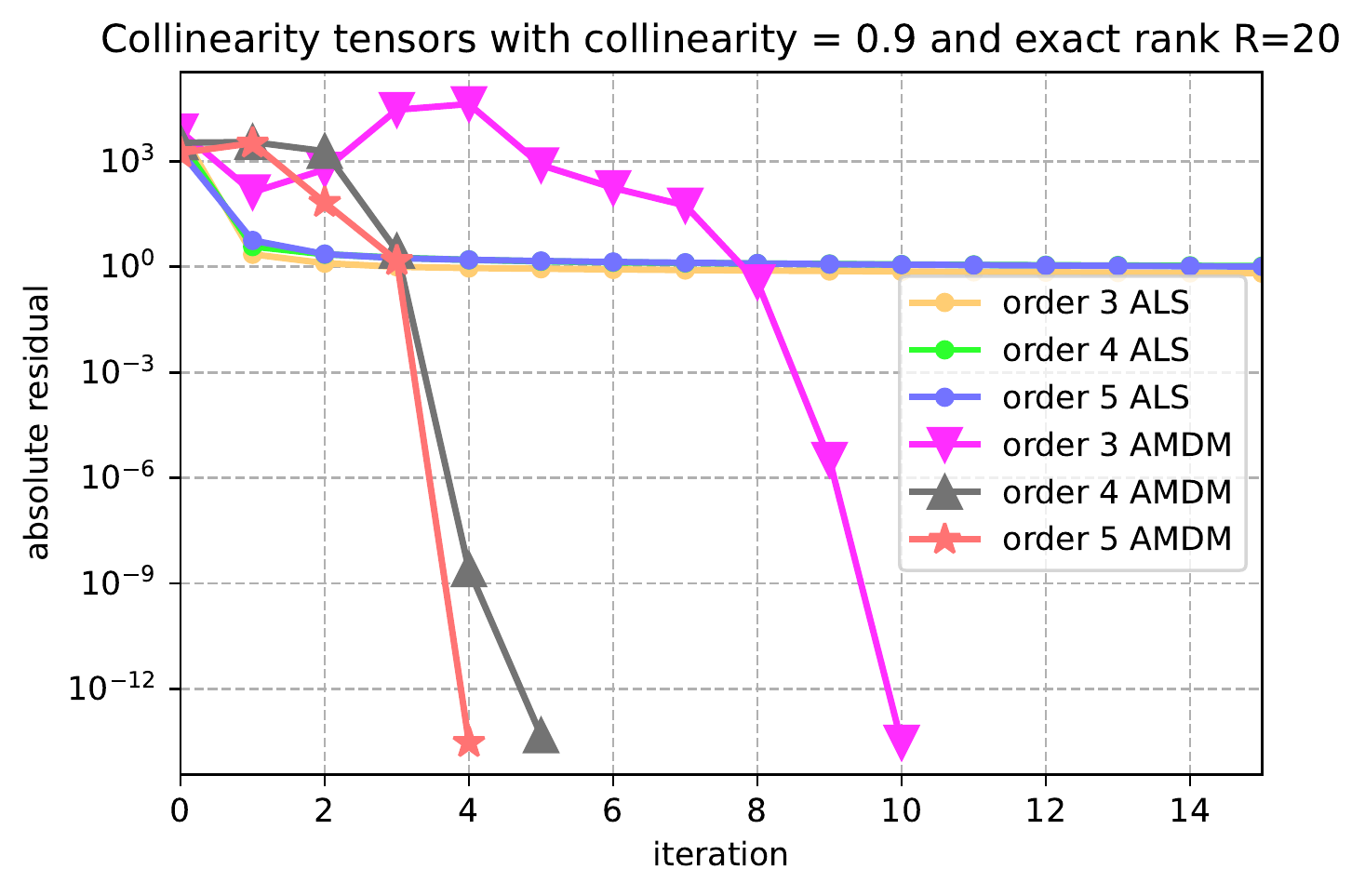}} \end{subfigure}
\caption{Superlinear convergence of AMDM algorithm for exact CPD}
\label{fig:higher_order_conv}
\end{figure*}

\begin{figure*}[t]
\centering
\begin{subfigure}[Rate of convergence] {\label{fig:roc}\includegraphics[width=0.50\textwidth, keepaspectratio]{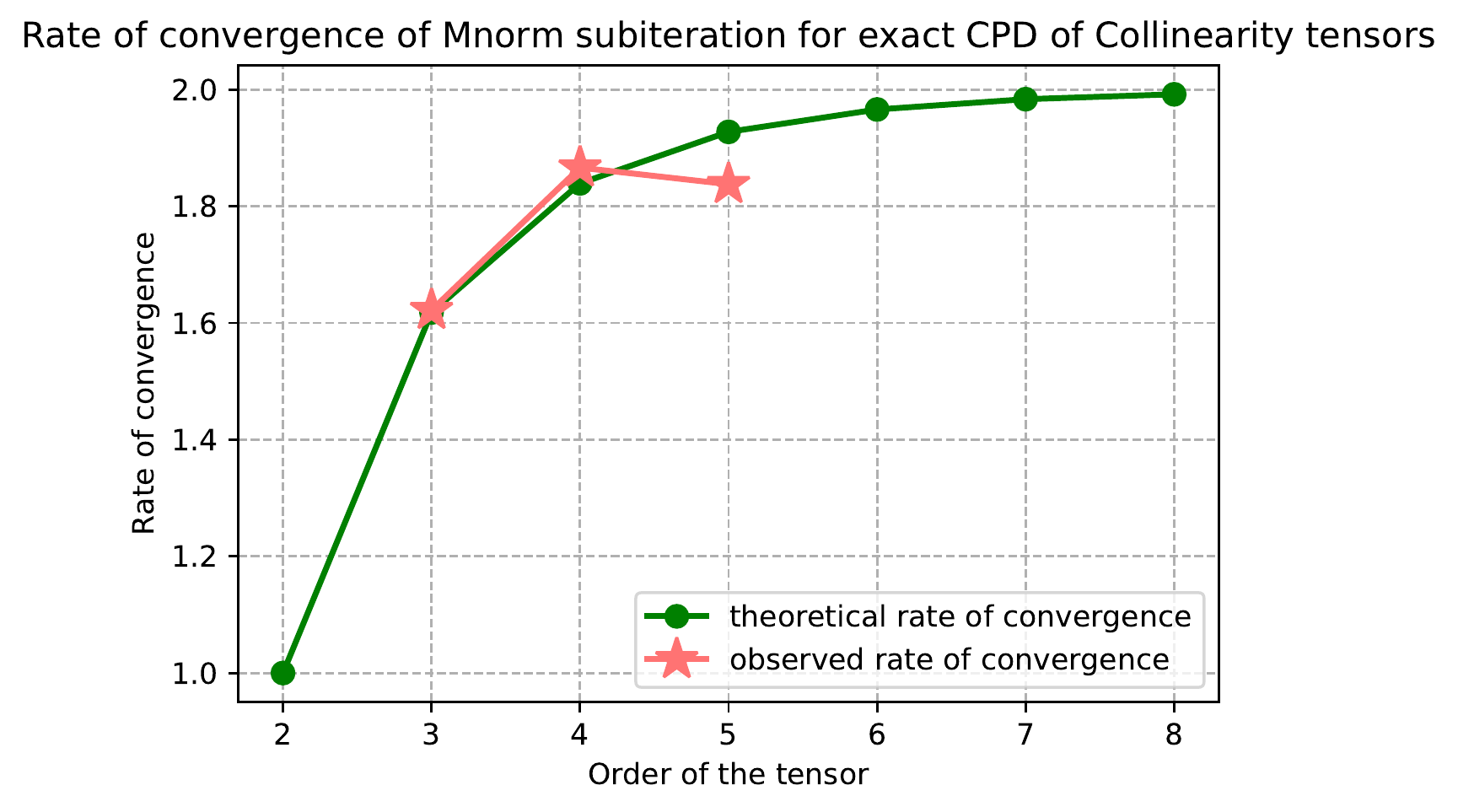}} \end{subfigure}
\begin{subfigure}[AMDM for rank larger than dimension]{\label{fig:rand_large}\includegraphics[width=0.45\textwidth, keepaspectratio]{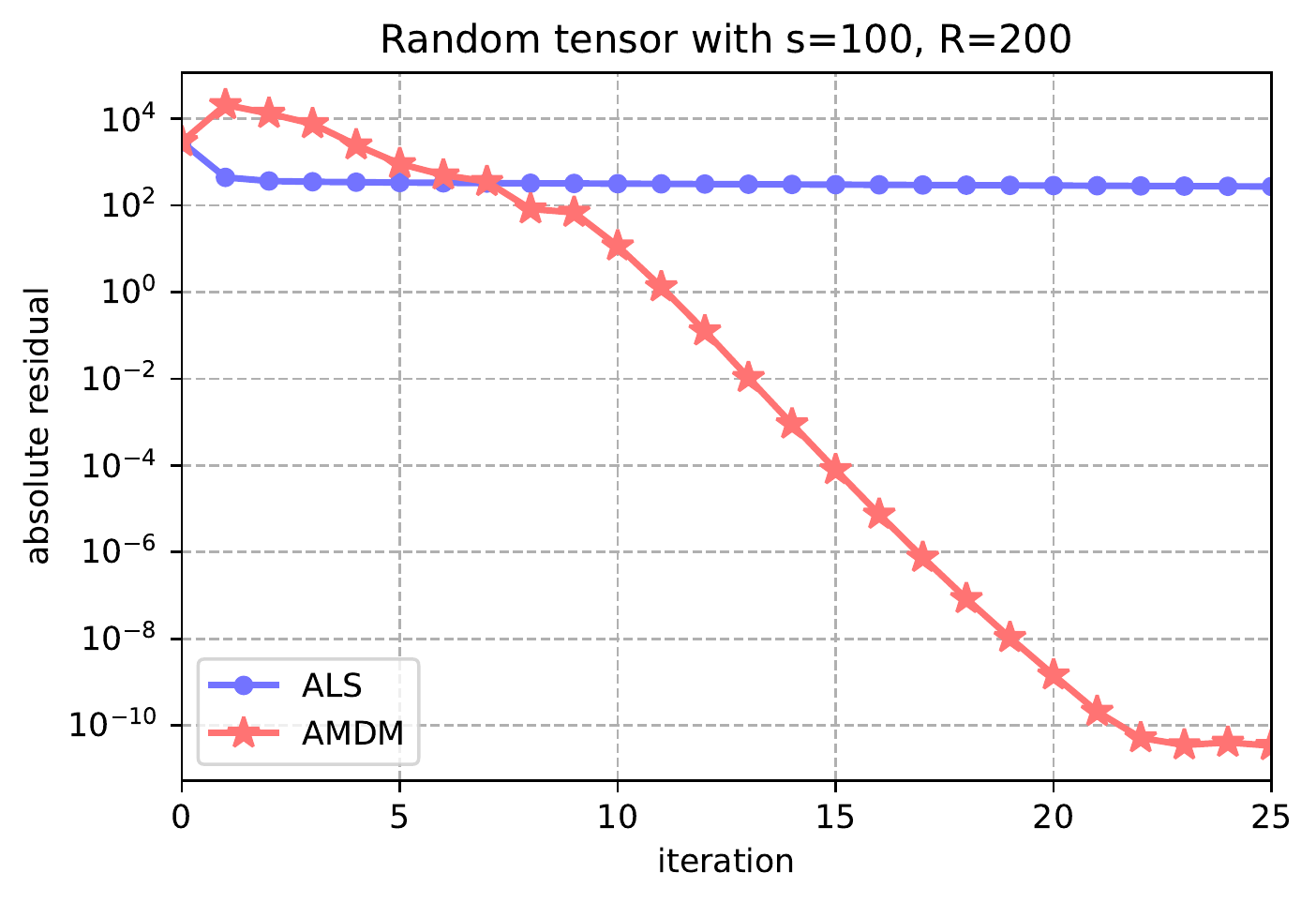}} \end{subfigure}
\caption{Rate of convergence of AMDM subiteration and linear convergence for large rank for exact CPD}
\label{fig:roc_and_first_ord_conv}
\end{figure*}

In Figure~\ref{fig:higher_order_conv}, we create equidimensional synthetic tensors of order $N$ with each mode length $s$ to follow $s^N=100^{3}$ with fixed CP rank $R=20$.  We initialize the factors with uniformly distributed random matrices for both the algorithms and plot the absolute residual for each iteration. For the \textit{Collinearity} tensors, collinearity value, i.e., $C$ is set to be $0.9$. We can observe superlinear convergence of Algorithm~\ref{alg:mnorm_basic} for both the cases, while ALS appears to be slowed down by the `swamp' phenomenon for the \textit{Collinearity} tensors. In Figure~\ref{fig:roc}, we plot
the empirical rate of convergence of a subiteration by using the relative residual after subiterations. Since order 2 is just matrix decomposition, AMDM and ALS algorithm become equivalent with a linear convergence rate. We can observe that the rate of local convergence for order 3 and 4 is consistent with what we observed in Section~\ref{sec:conv} with an error of $0.2\%$ and $1\%$ for order $3$ and $4$ respectively.  For order 5 and above, it is difficult to verify the rate of convergence as there does not exist two data points of subiterations where we do not convergence to machine precision and at the same time the assumptions for error to be small are satisfied.

In Figure~\ref{fig:rand_large}, we create a \textit{Random} equidimensional tensor of with mode length $s=100$ and CP rank $R=200$. We can observe that the Algorithm~\ref{alg:mnorm} converges linearly to the exact solution while taking only a few iterations to do so. Since  $R>s$, we observe a linear convergence rate which is consistent with our theoretical results. ALS makes slow progress for this case. A reason for that might again be related to the collinearity of the factors, since when $R>s$, collinearity is high and ALS is more likely to experience the `swamp' phenomenon~\cite{mitchell1994slowly}.

\subsection{Approximate CP decomposition}

\begin{figure*}[t]
\centering
\begin{subfigure}[Probability of convergence for equidimensional tensors with mode length$=10$, exact CP rank$=10$ and approximate rank$=5$] {\label{fig:prob_10}\includegraphics[width=0.45\textwidth, keepaspectratio]{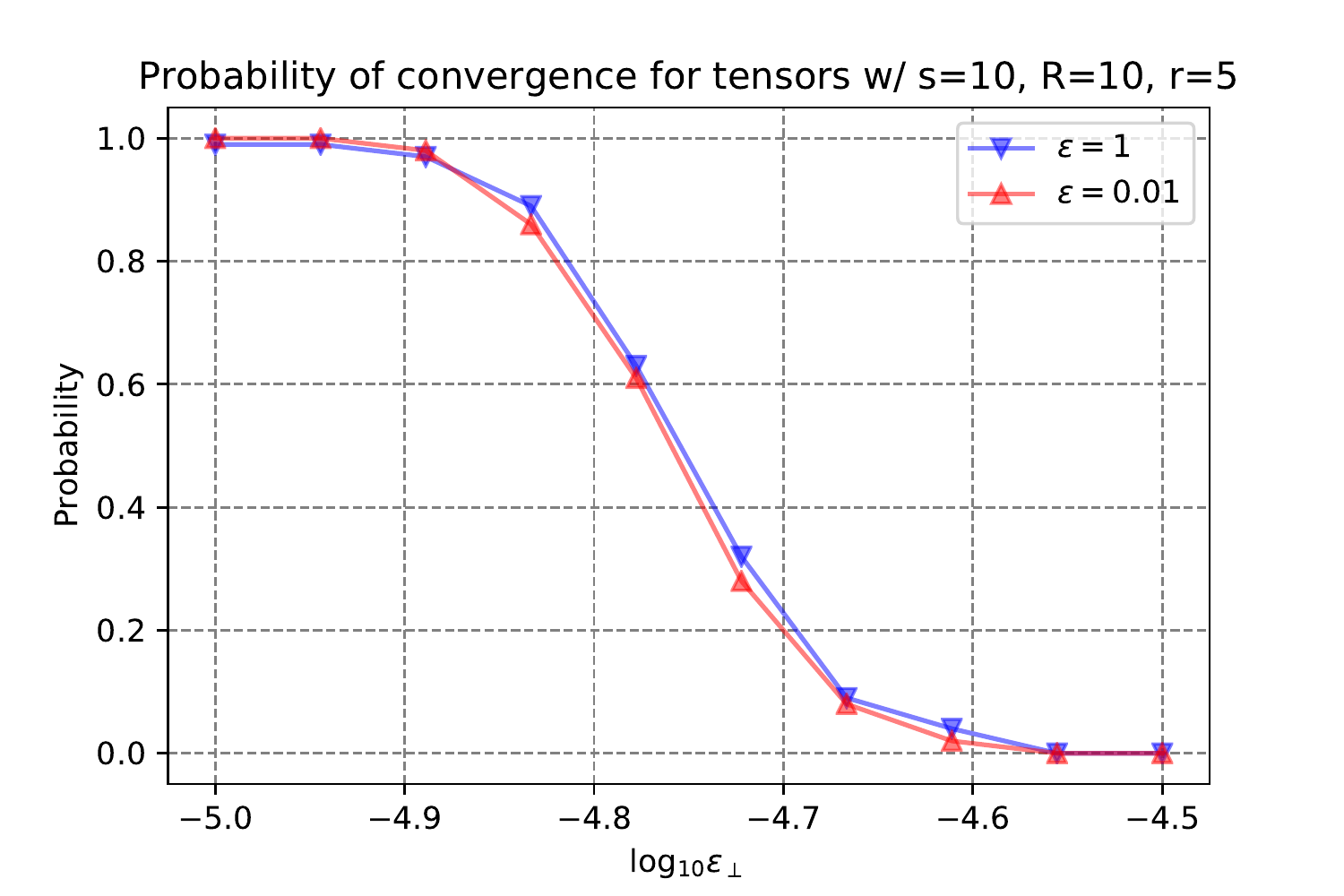}} \end{subfigure}
\begin{subfigure}[Probability of convergence for equidimensional tensors with mode length$=100$, exact CP rank$=100$ and approximate rank$=50$] {\label{fig:prob_100}\includegraphics[width=0.45\textwidth, keepaspectratio]{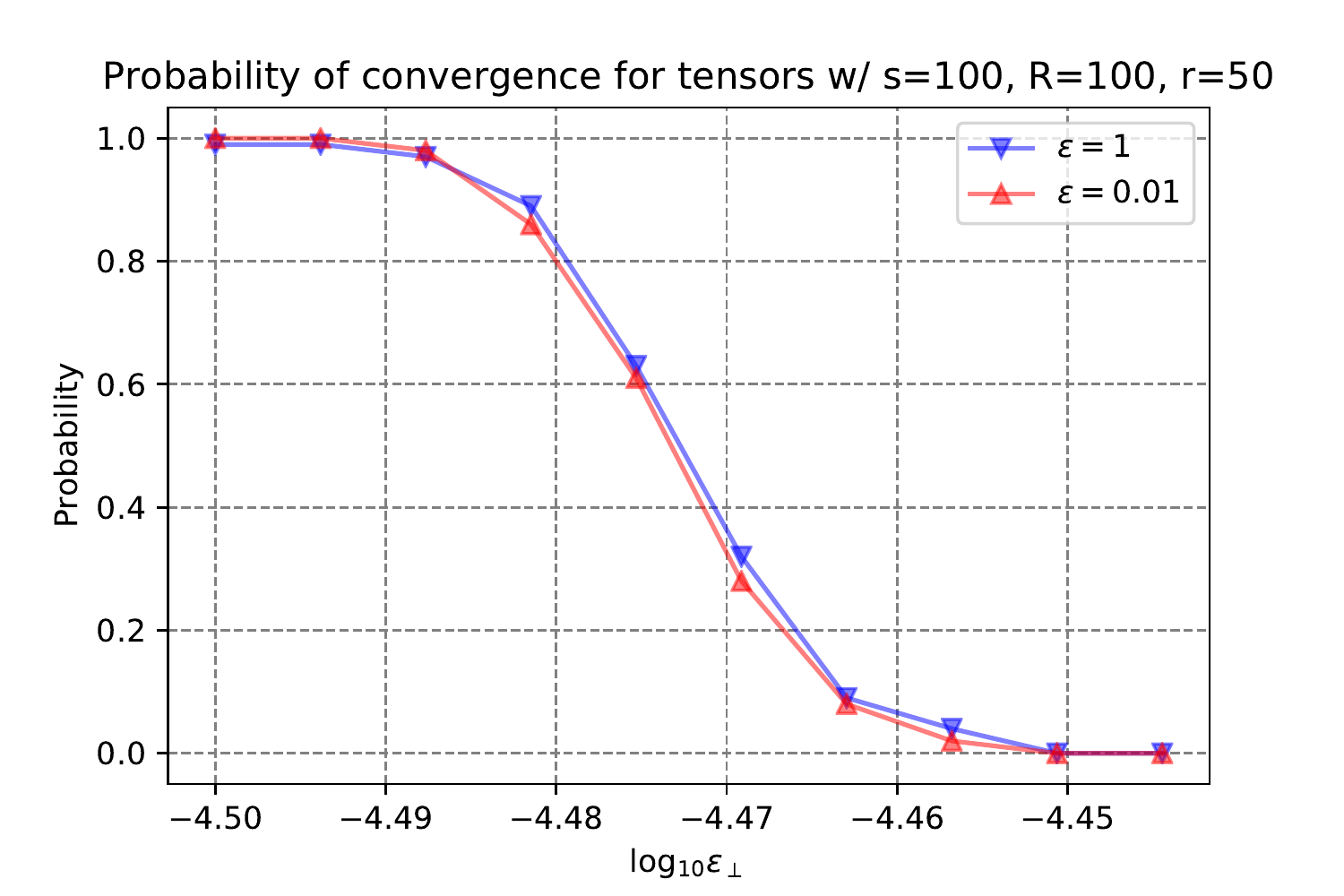}} \end{subfigure}
\caption{Probability of convergence for $100$ tensors with $5$ initial guesses in the setting as described in Lemma~\ref{lem:err_approx_mnorm}}
\label{fig:probability}
\end{figure*}

\begin{figure*}[t]
\centering
\begin{subfigure}[\textit{Collinearity} tensor with Gaussian noise fitness] {\label{fig:proba}\includegraphics[width=0.45\textwidth, keepaspectratio]{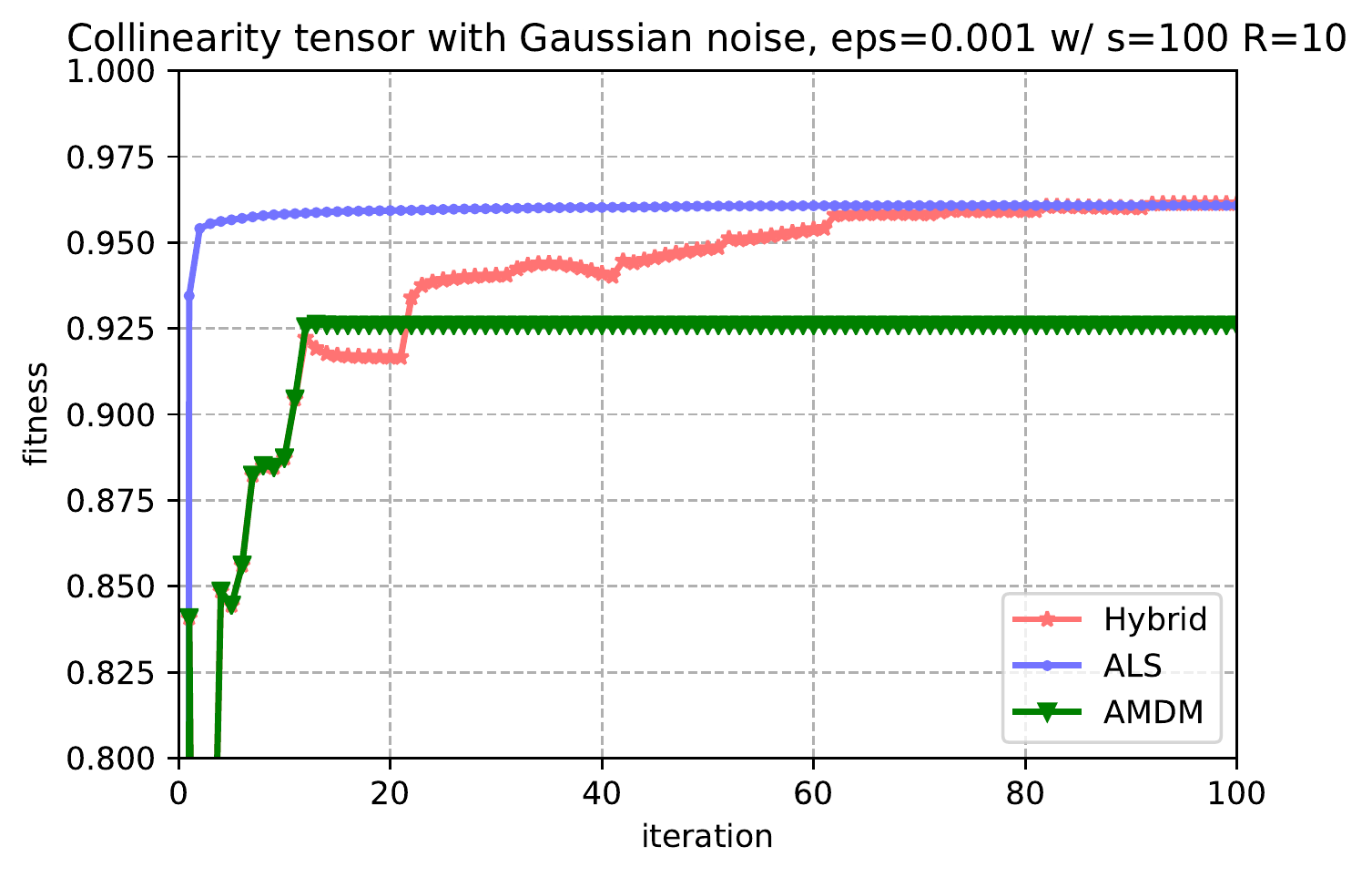}} \end{subfigure}
\begin{subfigure}[\textit{Collinearity} tensor with Gaussian noise condition number] {\label{fig:resplot}\includegraphics[width=0.45\textwidth, keepaspectratio]{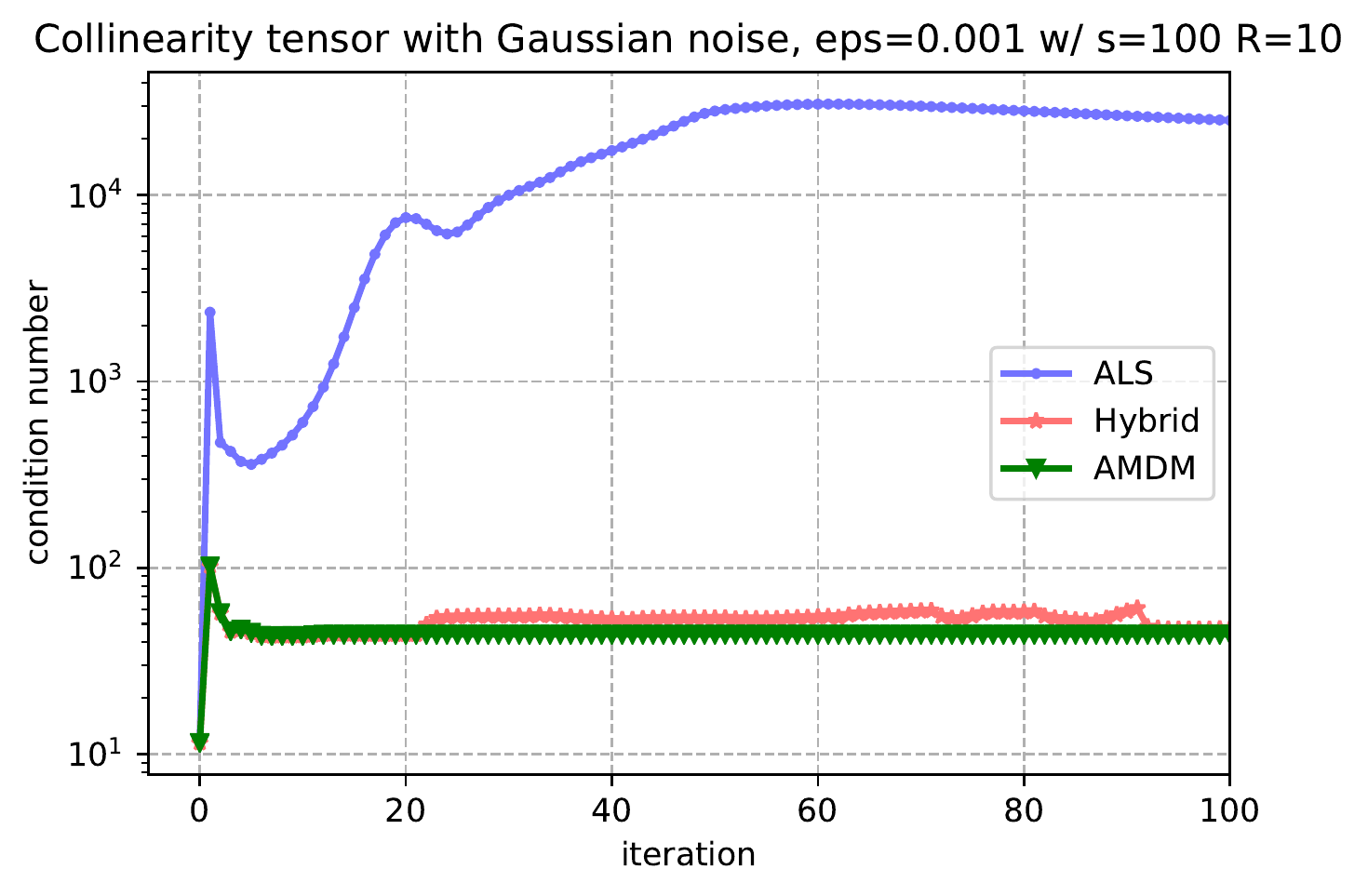}} \end{subfigure}
\caption{ Collinearity tensor of size $100 \times 100 \times 100$ with exact CP rank $R=10$ with added Gaussian noise with each entry distributed with mean $\mu =0$ and standard deviation $\sigma = 0.001$}
\label{fig:noisy}
\end{figure*}
\begin{figure*}[t]
\centering
\begin{subfigure}[SLEEP tensor fitness] {\label{fig:proba}\includegraphics[width=0.45\textwidth, keepaspectratio]{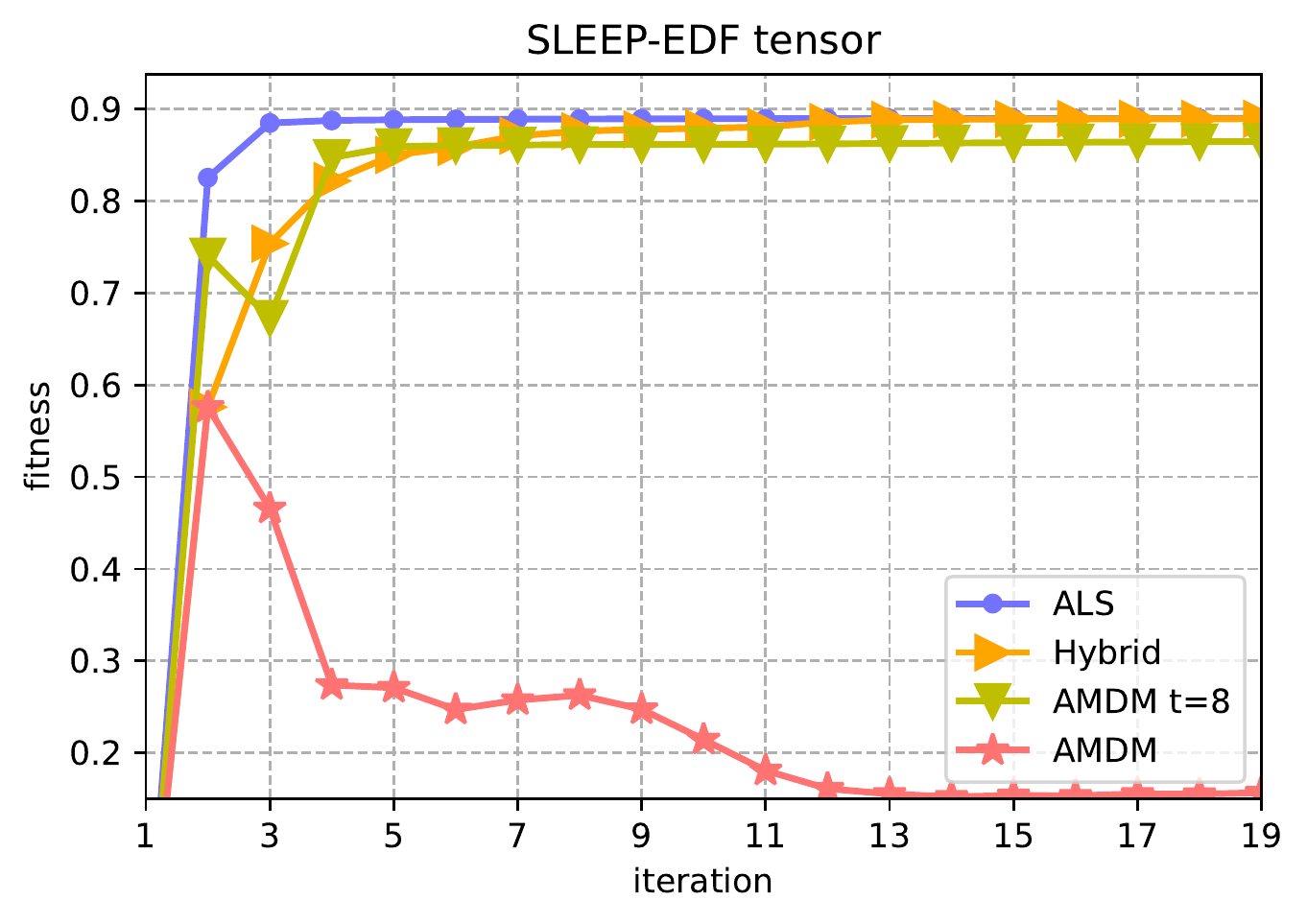}} \end{subfigure}
\begin{subfigure}[SLEEP tensor condition number] {\label{fig:resplot}\includegraphics[width=0.45\textwidth, keepaspectratio]{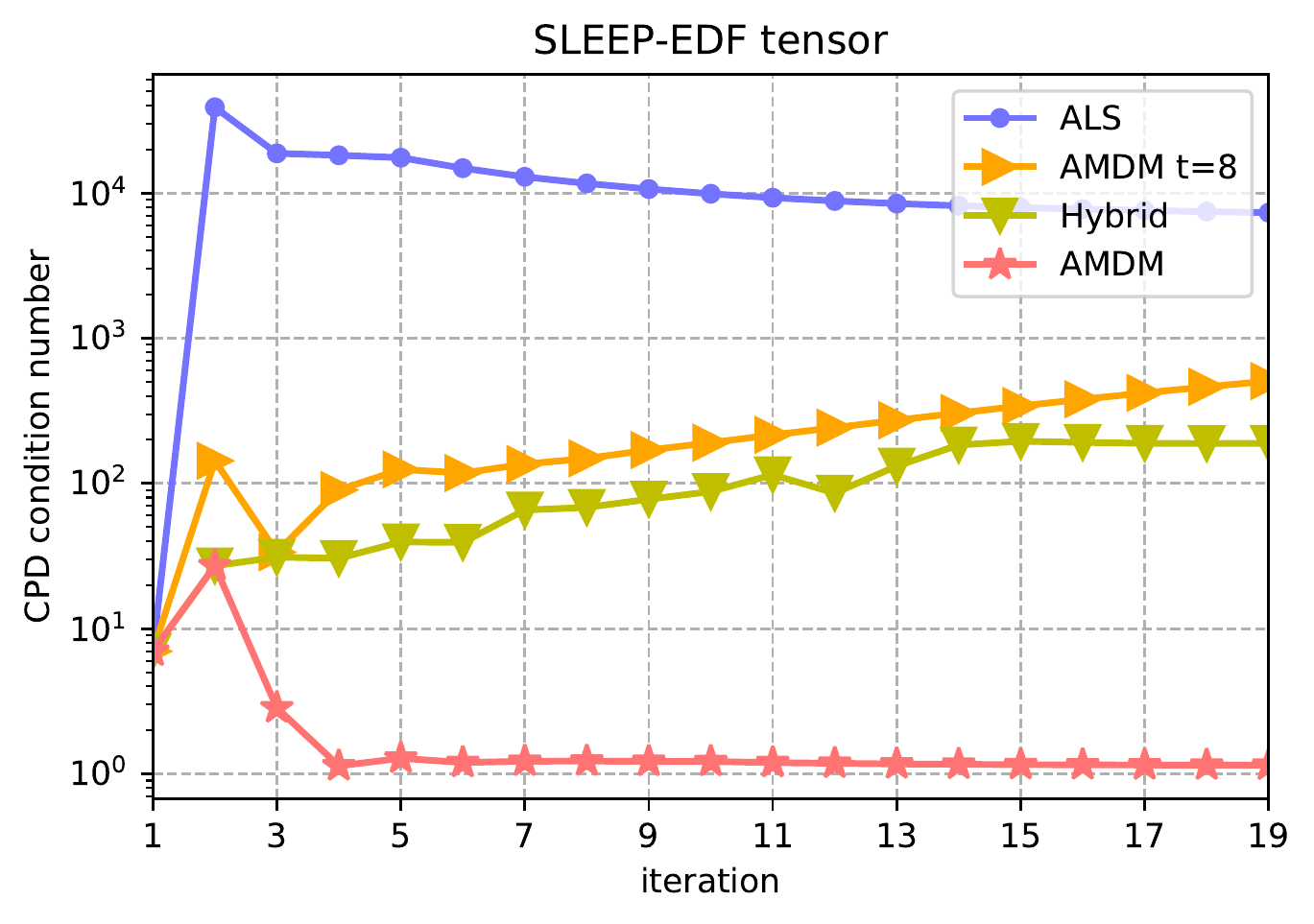}} \end{subfigure}
\caption{SLEEP-EDF tensor of dimensions $2048\times  14 \times 129 \times 86$ approximated with CP rank $R=10$, where $t$ is the singular value threshold in Algorithm~\ref{alg:mnorm}. }
\label{fig:SLEEP}
\end{figure*}
\begin{figure*}[t]
\centering
\begin{subfigure}[MGH tensor fitness] {\label{fig:proba}\includegraphics[width=0.45\textwidth, keepaspectratio]{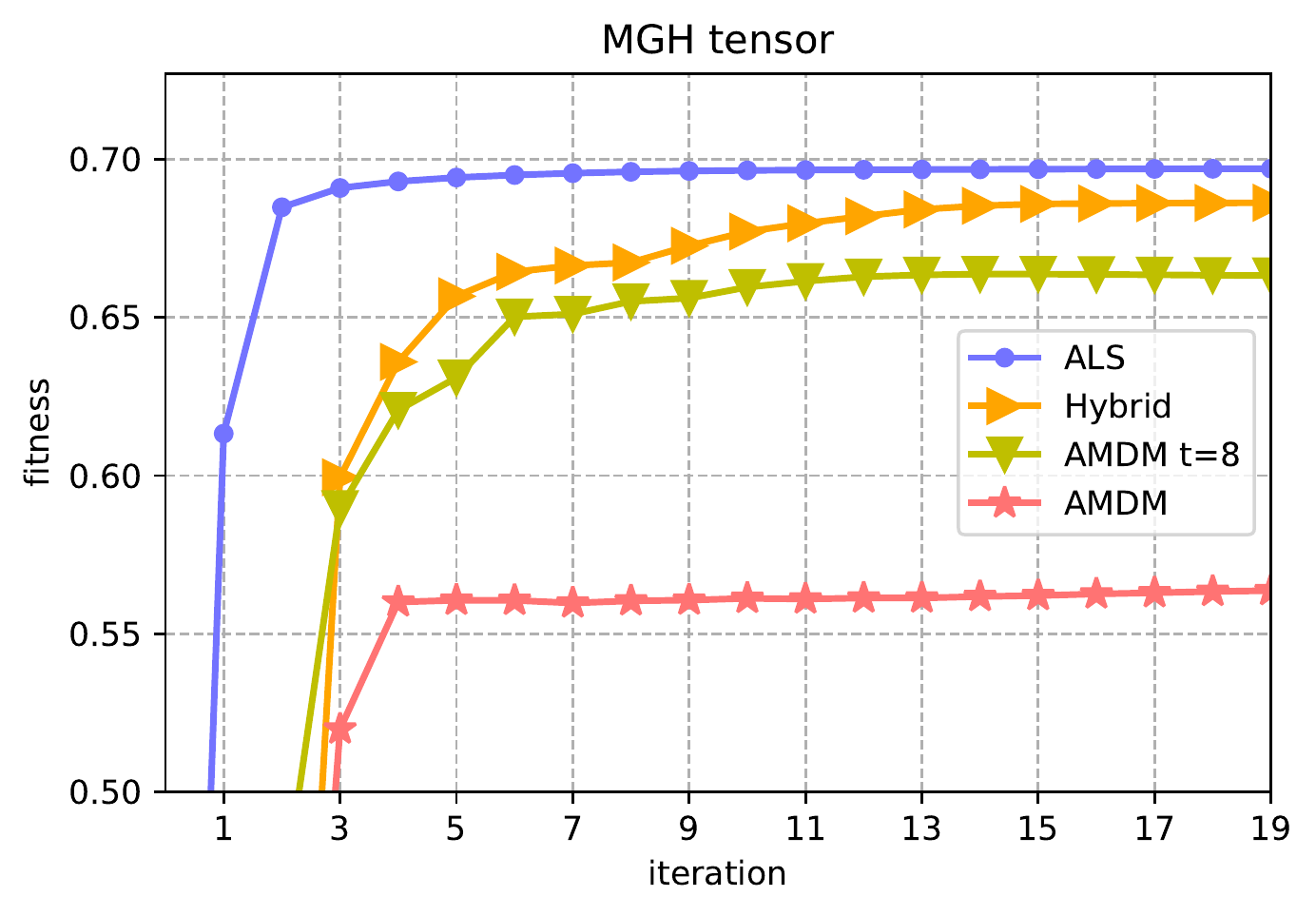}} \end{subfigure}
\begin{subfigure}[MGH tensor condition number] {\label{fig:resplot}\includegraphics[width=0.45\textwidth, keepaspectratio]{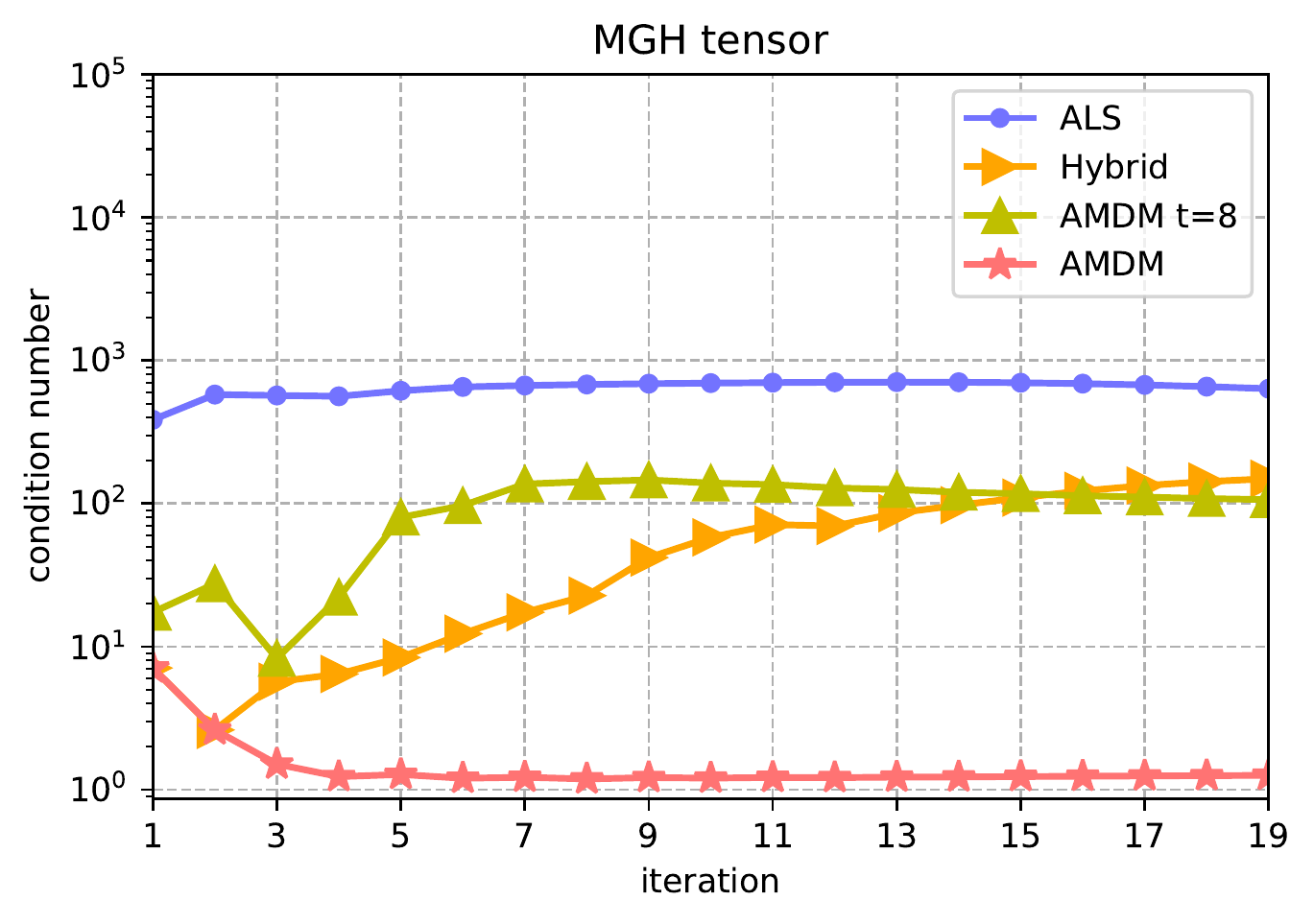}} \end{subfigure}
\caption{MGH tensor of dimensions $2048 \times 12 \times  257 \times 43$ approximated with CP rank $R=10$, where $t$ is the singular value threshold in Algorithm~\ref{alg:mnorm}}
\label{fig:MGH}
\end{figure*}
\begin{figure*}[t]
\centering
\begin{subfigure}[Amino acid tensor fitness] {\label{fig:proba}\includegraphics[width=0.45\textwidth, keepaspectratio]{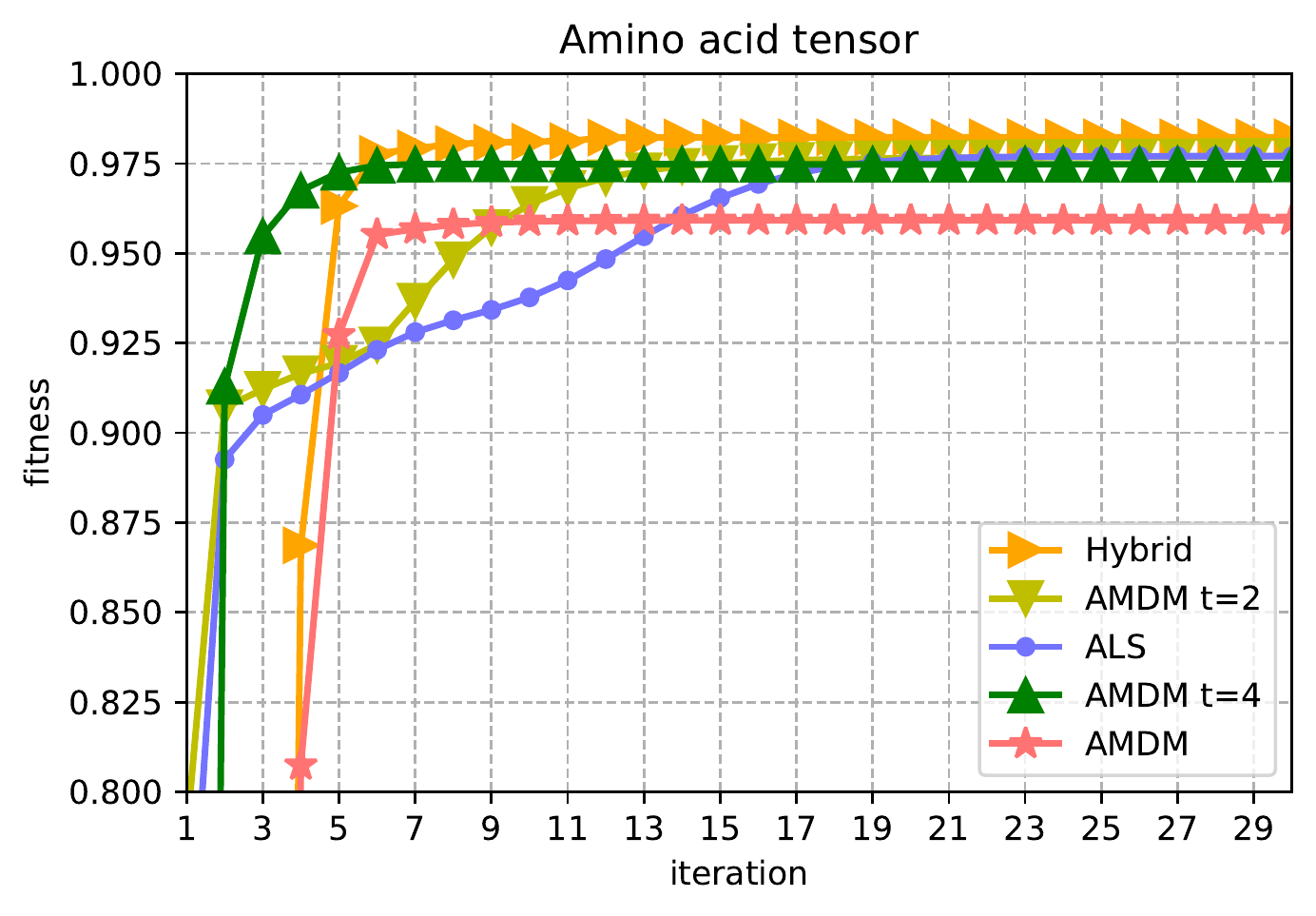}} \end{subfigure}
\begin{subfigure}[Amino acid tensor condition number] {\label{fig:resplot}\includegraphics[width=0.45\textwidth, keepaspectratio]{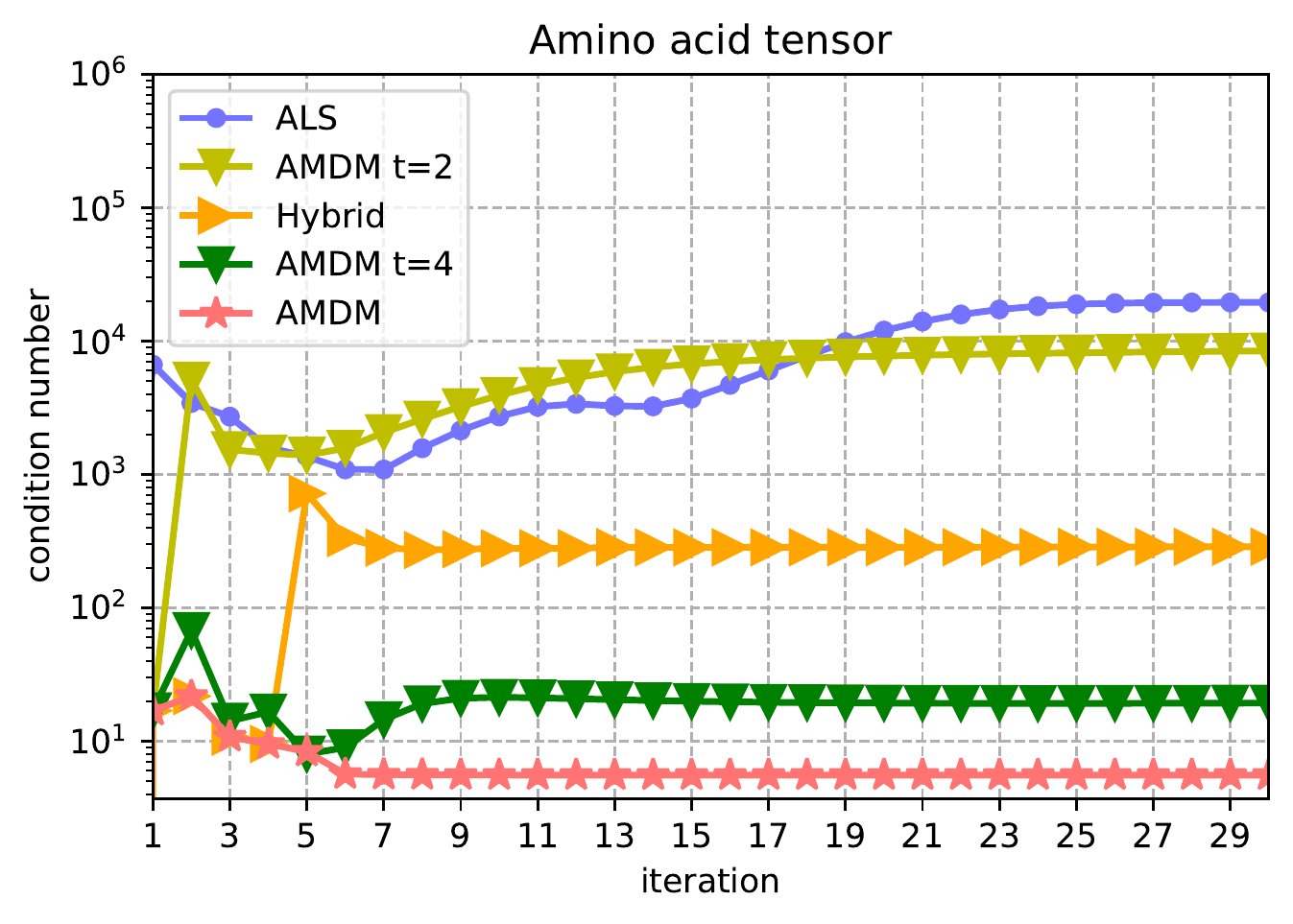}} \end{subfigure}
\caption{Amino acid tensor of dimensions $5 \times 61 \times 201$ approximated with CP rank $R=5$, where $t$ is the singular value threshold in Algorithm~\ref{alg:mnorm} }
\label{fig:amino}
\end{figure*}
\begin{figure}
    \centering
    \includegraphics[width=0.45\textwidth, keepaspectratio]{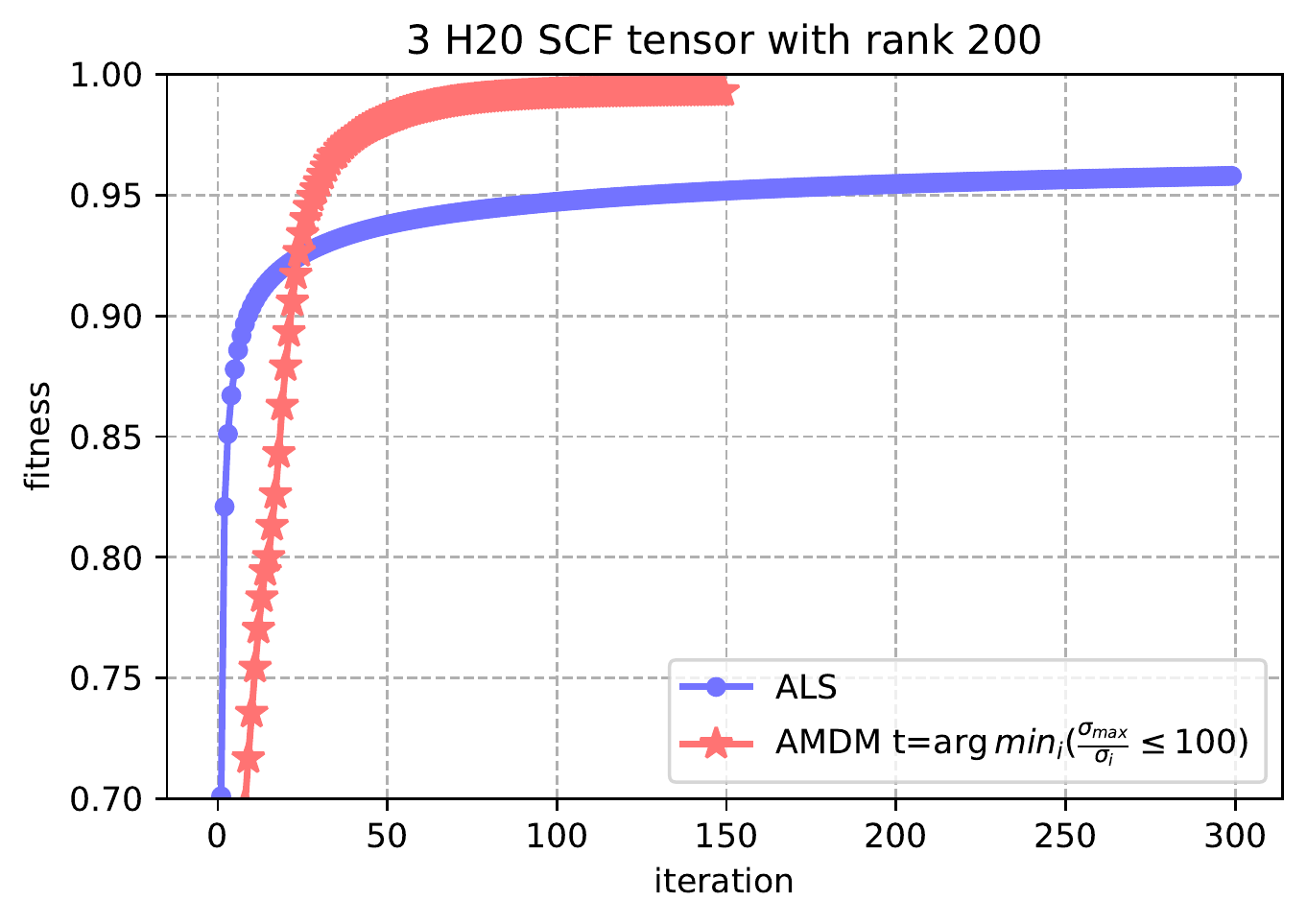}
    \caption{SCF tensor of size $339 \times 21 \times 21$ approximated with CP rank $R=200$,  where $t$ is the singular value threshold in Algorithm~\ref{alg:mnorm}.}
    \label{fig:scf_tensor}
\end{figure}

We plot the probability of convergence to the desired decomposition for synthetic tensors as described in~\ref{lem:err_approx_mnorm} with respect to the $\epsilon_\perp$ to verify our theoretical results in Figure~\ref{fig:probability}. We construct these tensors by constructing first $R/2$ columns of the factors with random matrices and then constructing the other half by projecting it onto the orthogonal complement of the column space of the first half and adding Gaussian noise of amplitude $\epsilon_\perp$ to the same. We construct $100$ such tensors and for each tensor we consider $5$ initial guesses which are $\epsilon$ away from the desired decomposition as described in~\ref{lem:err_approx_mnorm}. We plot the probability of convergence over these $100$ tensors by considering if atleast $1$ initial guess is within $10^{-9}$ of the desired factors. We observe that the probability of convergence to the desired decomposition is $1$ when the $\epsilon_\perp$ is small, irrespective of the size and $\epsilon$, thereby verifying that the first half of CP decomposition is a stationary point. The probability decreases as we increase $\epsilon_\perp$, i.e., the decomposition converges to different stationary points for these tensors.

We compare ALS and different variants of AMDM for computing approximate CP decomposition of synthetic tensors in Figure~\ref{fig:noisy} and application tensors that admit approximations with a low CP rank in Figure~\ref{fig:SLEEP},
Figure~\ref{fig:MGH}, Figure~\ref{fig:amino}, and Figure~\ref{fig:scf_tensor}. We plot the fitness and the condition number of the CP decomposition to compare ALS and variants of AMDM. The integer associated with AMDM $t= \#$ corresponds to the number of singular values inverted for each factor or best rank$-t$ approximation $\mat A_1$ in Equation~\eqref{eq:M_in_gen_AMDM}. The hybrid algorithm starts by using threshold $t=R$ in Algorithm~\ref{alg:mnorm}, i.e., starts with Algorithm~\ref{alg:mnorm_basic} and gradually decreases the threshold to $0$ to recover ALS algorithm.

In Figure~\ref{fig:noisy}, we compute a rank $10$ CP decomposition of the \textit{Collinearity} tensor with collinearity $C=0.9$ and exact CP rank $R=10$ with added Gaussian noise tensor. Each entry of the noise tensor is distributed normally with mean $\mu = 0$ and standard deviation $\sigma = 0.001$. We observe that the AMDM algorithm maintains a low condition number while reaching a high fitness for both the input tensor and the underlying tensor, while ALS algorithm reaches a higher fitness at the cost of highly ill conditioned decomposition.  For the hybrid algorithm, one less singular value is inverted after every $10$ iterations, leading to a decomposition with fitness as high as ALS and conditioning as good as the AMDM algorithm. We observe a similar behaviour for different dimensions and CP rank aproximations of the tensor, suggesting that there maybe multiple optimal decompositions for such problems.

In Figure~\ref{fig:SLEEP} and Figure~\ref{fig:MGH} , we compute the CP decomposition of the SLEEP-EDF tensor and MGH tensor with CP rank $R=10$. We consider several variants of hybrid Algorithm~\ref{alg:mnorm}. For the hybrid algorithm, one less singular value is inverted after every iteration. We clearly see a pattern in both the tensors that if lesser singular values are inverted then the fitness is higher and the CPD condition number is larger. We also see that the hybrid algorithm is able to achieve a fitness almost as high as ALS while maintaining a lower condition number.  The condition number of decomposition with the hybrid algorithm is about $3.2$x lower for the MGH tensor with an absolute difference in fitness being $0.01$ or $0.015$\%, whereas the condition number is about $34.6$x lower for the SLEEP tensor with an absolute difference fitness $0.0004$ with hybrid algorithm being more accurate.

In Figure~\ref{fig:amino}, we compute the CP decomposition of Amino acid tensor with rank $R=5$. We have similar observations for the fitness and condition numbers of variants of the AMDM algorithm and ALS. In this case, the hybrid algorithm and AMDM with $t=2$ achieve a better fitness than ALS. The maximum fitness for hybrid algorithm is $0.982$ whereas the maximum fitness for ALS is $0.977$. The condition number of ALS is about $69$ times higher than that of the hybrid algorithm. Note that the fitness for AMDM (all singular values inverted) is $0.959$ with a condition number equal to $5.56$ indicating that the factor matrices have almost orthogonal columns.

In Figure~\ref{fig:scf_tensor}, we compute CP decomposition of the SCF tensor with rank $R=200$ (exceeding 2 of the 3 tensor dimensions). We use a relative tolerance criteria for computing $t = \text{argmin}_i (\frac{\sigma_{\text{max}}}{\sigma_i} < 100)$ in the AMDM algorithm, i.e., singular values $\sigma_i$ are inverted only if $\frac{\sigma_{\text{max}}}{\sigma_i}< 100$, where $\sigma_{\text{max}}$ is the maximum singular value. The hybrid algorithm outperforms ALS in terms of fitness by reaching $0.993$ fitness in $150$ iterations whereas ALS reaches $0.96$ in $300$ iterations.

\section{Conclusion}
\label{sec:cnc}
In this work, we have proposed an alternative optimization algorithm, AMDM, to compute a CP decomposition of the tensor. This algorithm achieves superlinear local convergence for exact CP rank problems when CP rank is smaller than or equal to all the mode lengths of the tensor with the same asymptotic computational cost as that of ALS. For approximating a tensor via CP decomposition, we theoretically show that the algorithm locally converges to the stationary points of~\eqref{eq:f_tsr} for tensors with special CP structure. Although, the existence of these stationary points for any tensor is an open problem, we empirically confirm that the AMDM algorithm converges to these stationary points for various tensors. Viewing the algorithm as minimizing a Mahalanobis distance helps in generalization of the method for CP rank larger than the mode lengths and interpolate between AMDM and ALS algorithms. We also formulate an efficient way to compute the CPD condition number to track the condition of the decomposition throughout the algorithm. Our numerical experiments confirm that interpolation of algorithms between AMDM and ALS leads to a better conditioned decomposition without significant difference in fitness as compared to ALS for synthetic as well as most of the tested real world tensors. We provide an intuitive reasoning of this phenomenon and leave the detailed analysis as a future direction of research.

\section{Acknowledgments}
\label{sec:ack}
The authors would like to thank Ardavan (Ari) Afshar for detailed discussions about his work on minimizing Wasserstein distance between tensors from which this work is derived. The authors would also like to thank Jimeng Sun, Cheng Qian and Chaoqi Yang for having fruitful discussions and providing datasets which motivated this work. Navjot Singh and Edgar Solomonik were supported by the US NSF OAC SSI program, award No.\ 1931258. 

\begin{appendix}

\section{Computing the Condition Number of a CP Decomposition}
\label{sec:app}
It has been shown that the CPD condition number is the reciprocal of the smallest singular value of a matrix called Terracini's matrix. This matrix consists of the orthogonal basis for the tangent space of each of the rank-$1$ components of the reconstructed tensor. We will refer to the normalized condition number as the condition number of CPD and we refer the reader to~\cite{breiding2018condition,vannieuwenhoven2017condition} for details about how a notion of condition number of a CP decomposition is defined and derived. Consider an equidimensional order $3$ real tensor $\tsr{X}$ with mode length $s$. Let the CPD approximation of rank $R$ be given by $[\![ \vcr \lambda ; \mat{A}, \mat{B}, \mat{C} ]\!]$, then the Terracini's matrix $\mat U$ is $\mat U = [\mat U_1 \ldots \mat U_R]$, where $\forall i \in \{1,\ldots, R\}$, 
\begin{align*}
    \mat U_i = [\vcr a_i \otimes \vcr b_i \otimes \vcr c_i \quad \mat Q^\perp_{\vcr a_i} \otimes \vcr b_i \otimes \vcr c_i \quad \vcr a_i  \otimes \mat Q^\perp_{\vcr b_i}  \otimes \vcr c_i \quad  \vcr a_i \otimes \vcr b_i \otimes \mat Q^\perp_{\vcr c_i}],
\end{align*}
and $\mat Q^\perp_{\vcr a_i} \in \mathbb{R}^{s \times (s-1)}$ is an orthogonal basis of the orthogonal complement of $\vcr a_i$, and $\mat Q^\perp_{\vcr b_i}$, and $\mat Q^\perp_{\vcr c_i}$ are defined similarly. Consequently, the Terrracini's matrix is of size $s^3 \times R(3s- 2)$, and the computational cost of computing the smallest singular value via a Krylov subspace method is $O(s^5R^2)$. For an order $N$ tensor, this cost is $O(N^2s^{N+2}R^2)$ and therefore expensive to compute for decompositions with moderately large mode lengths.

The cost of computing the condition number can be decreased significantly for when rank of the CP decomposition is lesser than all the mode lengths of the input tensor, i.e., if $R < s$. Assume that $R \leq s$, then since the condition number is invariant to orthogonal transformations~\cite{breiding2018condition}, for a CPD of an order $3$ tensor, \[\kappa\big([\![ \vcr \lambda ; \mat{A}, \mat{B}, \mat{C} ]\!]\big) = \kappa\big([\![ \vcr \lambda ; \mat Q_{\mat A}^T \mat{A}, \mat Q^T_{\mat B}\mat{B}, \mat Q^T_{\mat C}\mat{C} ]\!]\big),\]
where $\mat Q_{\mat A} \in \mathbb{R}^{s \times s} = [ \mat Q^{(1)}_{\mat A} \mat Q^{(2)}_{\mat A}]$, and
the columns of
$\mat Q^{(1)}_{\mat A} \in \mathbb{R}^{s \times R}$ 
are an orthogonal basis of the column space of $\mat A$, while the columns of $\mat Q^{(2)}_{\mat A} \in \mathbb{R}^{s \times (s-R) }$ are an orthogonal basis for the orthogonal complement of the column space of $\mat A$. We  define $\mat Q_{\mat B} = [ \mat Q^{(1)}_{\mat B} \mat Q^{(2)}_{\mat B}]$ and $\mat Q_{\mat C} =[ \mat Q^{(1)}_{\mat C} \mat Q^{(2)}_{\mat C}]$ similarly. The transformed Terracini's matrix $\bar{\mat U} = [ \bar{\mat U_1} \ldots \bar{\mat U_R}]$, where $\forall i \in \{1,\ldots, R\}$,
\begin{align*}
    \bar{\mat U}_i = 
    [
    \underbrace{\mat Q^T_{\mat A}\vcr a_i \otimes \mat Q^T_{\mat B}\vcr b_i \otimes \mat Q^T_{\mat C}\vcr c_i}_{\bar{\mat U}_i^{(1)}} \quad
    \underbrace{\bar{\mat Q}^\perp_{\vcr a_i} \otimes \mat Q^T_{\mat B}\vcr b_i \otimes \mat Q^T_{\mat C}\vcr c_i}_{\bar{\mat U}_i^{(2)}}
    \quad
    \underbrace{\mat Q^T_{\mat A}\vcr a_i  \otimes \bar{\mat Q}^\perp_{\vcr b_i}  \otimes \mat Q^T_{\mat C}\vcr c_i}_{{\bar{\mat U}_i^{(3)}}}
    \quad
    \underbrace{\mat Q^T_{\mat A}\vcr a_i \otimes \mat Q^T_{\mat B}\vcr b_i \otimes \bar{\mat Q}^\perp_{\vcr c_i}}_{{\bar{\mat U}_i^{(4)}}}
    ],
\end{align*}
where $\bar{\mat Q}^\perp_{\vcr a_i} = \mat Q^T_{\mat A}\mat Q^\perp_{\vcr a_i} $ is an orthogonal basis of the orthogonal complement of $\mat Q^T_{\mat{A}}\vcr a_i$, and $\bar{\mat Q}^\perp_{\vcr b_i}$, and $\bar{\mat Q}^\perp_{\vcr c_i}$ are defined similarly.
Note that $\bar{\mat U}_i^{(j)}\bar{\mat U}_i^{(k)}=\mat 0$ for $j\neq k$, since $\bar{\mat Q}^\perp{}^T\vcr{a}_i\mat{Q}_{\mat{A}}^T\vcr{a}_i=0$.
Consequently,
\[\sigma_{\min}(\mat U)  = \min_{j\in\{1,2,3,4\}} \sigma_{\min}([\mat U_1^{(j)}\quad  \ldots \quad\mat U_R^{(j)}]).\]
After this transformation, we can obtain a reduced form of smaller dimensions each of the four matrices to compute the condition number more efficiently.
Note that, $\mat Q_{\mat A}^T \vcr{a}_i = \begin{bmatrix} \mat Q_{\mat A}^{(1)}{}^T \vcr{a}_i \\ \vcr 0 \end{bmatrix}$  and similar for $\mat B$ and $\mat C$, we have that
\[\sigma_{\min}([\mat U_1^{(1)}\quad \ldots\quad \mat U_R^{(1)}]) = \sigma_{\min}([\mat Q^{(1)}_{\mat A}{}^T\vcr a_1\otimes \mat Q^{(1)}_{\mat B}{}^T\vcr b_1 \otimes \mat Q^{(1)}_{\mat C}{}^T\vcr c_1 \quad \ldots \quad \mat Q^{(1)}_{\mat A}{}^T\vcr a_R\otimes \mat Q^{(1)}_{\mat B}{}^T\vcr b_R \otimes \mat Q^{(1)}_{\mat C}{}^T\vcr c_R ]).\]
The reduced matrix above is of dimension $R^3\times R$ instead of $s^3\times R$.
Further, we can choose the columns $\mat Q^\perp_{\vcr a_i}$ so that $\bar{\mat Q}^\perp_{\vcr a_i} = \mat Q^T_{\mat A}\mat Q^\perp_{\vcr a_i} = \begin{bmatrix} \mat Q^\perp_{\mat Q^{(1)}_{\mat A}{}^T\vcr a_1} & \mat 0 \\ \mat 0 & \mat I \end{bmatrix}$, and similar for $\mat B$ and $\mat C$.
Consequently, for $j=2$,
\begin{align*}
\sigma_{\min}([\mat U_1^{(2)}\quad \cdots \quad \mat U_R^{(2)}]) 
&= \sigma_{\min}\Bigg(\Bigg[\begin{bmatrix} \mat Q^\perp_{\mat Q^{(1)}_{\mat A}{}^T\vcr a_1} & \mat 0 \\ \mat 0 & \mat I \end{bmatrix} \otimes \mat Q^{(1)}_{\mat B}{}^T\vcr b_1 \otimes \mat Q^{(1)}_{\mat C}{}^T\vcr c_1\quad \ldots \Bigg]\Bigg) \\
&= \min\{\sigma_{\min}([\mat Q^\perp_{\mat Q^{(1)}_{\mat A}{}^T\vcr a_1} \otimes \mat Q^{(1)}_{\mat B}{}^T\vcr b_1 \otimes \mat Q^{(1)}_{\mat C}{}^T\vcr c_1\ \  \ldots]), \sigma_{\min}([\mat Q^{(1)}_{\mat B}{}^T\vcr b_1 \otimes \mat Q^{(1)}_{\mat C}{}^T\vcr c_1\ \  \ldots])\} \\
&= \sigma_{\min}([\mat Q^\perp_{\mat Q^{(1)}_{\mat A}{}^T\vcr a_1} \otimes \mat Q^{(1)}_{\mat B}{}^T\vcr b_1 \otimes \mat Q^{(1)}_{\mat C}{}^T\vcr c_1\quad \ldots \quad \mat Q^\perp_{\mat Q^{(1)}_{\mat A}{}^T\vcr a_R}\otimes \mat Q^{(1)}_{\mat B}{}^T\vcr b_R \otimes \mat Q^{(1)}_{\mat C}{}^T\vcr c_R ]),
\end{align*}
and similar for $j=3,4$.
The dimensions of the above reduced matrix are $R^3\times R(R-1)$, hence a direct computation of the singular value decomposition can be used to compute the condition number with cost $O(R^7)$.

All the above arguments can be generalized to an order $N$ non equidimensional tensor. Therefore, we showed that the condition number of CPD is invariant to the following transformation
\[\kappa\big([\![ \vcr \lambda ; \mat{A}_1 \ldots , \mat A_N ]\!]\big) = \kappa\big([\![ \vcr \lambda ; \mat Q^{(1)}_{\mat A_1}{}^T \mat{A}_1, \ldots , \mat Q^{(N)}_{\mat A_N}{}^T\mat{A}_N ]\!]\big),\]
where $\forall i \in \{1,\ldots, N \}$, columns of $\mat Q^{(1)}_{\mat A_i}$ are an orthonormal basis of the column space of $\mat A_i$.

\end{appendix}

\bibliographystyle{abbrv}
\bibliography{paper}

\end{document}